\documentclass{article}

\usepackage{arxiv}

\usepackage[utf8]{inputenc} 
\usepackage[T1]{fontenc}    
\usepackage{hyperref}       
\usepackage{url}            
\usepackage{booktabs}       
\usepackage{amsfonts}       
\usepackage{nicefrac}       
\usepackage{microtype}      
\usepackage{lipsum}		
\usepackage{graphicx}
\usepackage{natbib}
\usepackage{doi}

\usepackage{amsmath}
\usepackage[ruled]{algorithm2e}
\usepackage{mathtools}
\usepackage{booktabs,makecell,multirow}
\usepackage{algorithmic}
\usepackage{subfigure}

\usepackage{balance} 
\usepackage{color}
\definecolor{mygreen}{RGB}{70,102,94}
\definecolor{myyellow}{RGB}{234,178,0}
\definecolor{myper}{RGB}{128,82,110}
\definecolor{myred}{RGB}{192,0,0}
\definecolor{myblue}{RGB}{216,223,224}
\definecolor{mygreenshade}{RGB}{235,243,228}
\definecolor{myredshade}{RGB}{250,237,233}
\definecolor{myyellowshade}{RGB}{254,248,232}
\usepackage{bbm}
\usepackage{textcomp}
\usepackage{amssymb}
\usepackage{threeparttable}

\newtheorem{definition}{Definition}
\newtheorem{lemma}{Lemma}
\newtheorem{proof}{Proof}

\title{Causal Discovery from Time-Series Data with Short-Term Invariance-Based Convolutional Neural Networks}

\author{
    Rujia Shen \\
	Faculty of Computing\\ Harbin Institute of Technology\\ Harbin, Heilongjiang, China\\
	\texttt{shenrujia@stu.hit.edu.cn} \\
    \And
    Boran Wang\\
    The Artificial Intelligence Institute\\ Harbin Institute of Technology\\ Shenzhen, Guangdong, China\\
	\texttt{wangboran@hit.edu.cn} \\
    \And
    Chao Zhao\\
    The Department of Computer Science\\ The University of North Carolina at Chapel Hill\\ North Carolina, USA\\
	\texttt{zhaochao@cs.unc.edu} \\
    \And
    Yi Guan \\
	Faculty of Computing\\ Harbin Institute of Technology\\ Harbin, Heilongjiang, China\\
	\texttt{guanyi@hit.edu.cn}\\
    \And
    Jingchi Jiang\\
    The Artificial Intelligence Institute\\ Harbin Institute of Technology\\ Harbin, Heilongjiang, China\\
	\texttt{jiangjingchi@hit.edu.cn} \\
}

\renewcommand{\shorttitle}{\textit{arXiv} Template}

\hypersetup{
pdftitle={A template for the arxiv style},
pdfsubject={q-bio.NC, q-bio.QM},
pdfauthor={David S.~Hippocampus, Elias D.~Striatum},
pdfkeywords={First keyword, Second keyword, More},
}

\begin{document}
\maketitle

\begin{abstract}
	Causal discovery from time-series data aims to capture both intra-slice (contemporaneous) and inter-slice (time-lagged) causality between variables within the temporal chain, which is crucial for various scientific disciplines. Compared to causal discovery from non-time-series data, causal discovery from time-series data necessitates more serialized samples with a larger amount of observed time steps. To address the challenges, we propose a novel gradient-based causal discovery approach STIC, which focuses on \textbf{S}hort-\textbf{T}erm \textbf{I}nvariance using \textbf{C}onvolutional neural networks to uncover the causal relationships from time-series data. Specifically, STIC leverages both the short-term time and mechanism invariance of causality within each window observation, which possesses the property of independence, to enhance sample efficiency. Furthermore, we construct two causal convolution kernels, which correspond to the short-term time and mechanism invariance respectively, to estimate the window causal graph. To demonstrate the necessity of convolutional neural networks for causal discovery from time-series data, we theoretically derive the equivalence between convolution and the underlying generative principle of time-series data under the assumption that the additive noise model is identifiable. Experimental evaluations conducted on both synthetic and FMRI benchmark datasets demonstrate that our STIC outperforms baselines significantly and achieves the state-of-the-art performance, particularly when the datasets contain a limited number of observed time steps. Code is available at \url{https://github.com/HITshenrj/STIC}.
\end{abstract}


\section{Introduction}

Causality behind time-series data plays a significant role in various aspects of everyday life and scientific inquiry. Questions like ``What factors in the past have led to the current rise in blood glucose?" or ``How long will my headache be alleviated if I take that pill?" require an understanding of the relationships among observed variables, such as the relation between people's health status and their medical interventions \citep{cowls2015causation,pawlowski2020deep}. People usually expect to find cyclical and invariant principles in a changing world, which we call causal relationships \citep{chan2024fuzzy,entner2010causal}. These relationships can be represented as a directed acyclic graph (DAG), where nodes represent observed variables and edges represent causal relationships between variables with time lags. This underlying graph structure forms the factual foundation for causal reasoning and is essential for addressing such queries \citep{pearl2009causality}. 

Current causal discovery approaches utilize intra-slice and inter-slice information of time-series data, leveraging techniques such as conditional independence, smooth score functions, and auto-regression. These methods can be broadly classified into three categories: Constraint-based methods \citep{entner2010causal, runge2019detecting, runge2020discovering}, Score-based methods \citep{pamfil2020dynotears}, and Granger-based methods \citep{nauta2019causal,cheng2022cuts,cheng2023cuts+}.\textbf{Constraint-based methods} rely on conditional independence tests to infer causal relationships between variables. These methods perform independence tests between pairs of variables under different conditional sets to determine whether a causal relation exists. However, due to the difficulty of sampling, real-world data often suffers from the limited length of observed time steps, making it challenging for statistical conditional independence tests to fully capture causal relationships \citep{zhang2011kernel, zhang2023extending}. Additionally, these methods often rely on strong yet unrealistic assumptions, such as Gaussian noise, when searching for statistical conditional independence \citep{spirtes2016causal, wang2017efficient}. \textbf{Score-based methods} regard causal discovery as a constrained optimization problem using augmented Lagrangian procedures. They assign a score function that captures properties of the causal graph, such as acyclicity, and minimize the score function to identify potential causal graphs. While these methods offer simplicity in optimization, they relying heavily on acyclicity regularization and often lack guarantees for finding the correct causal graph, potentially leading to suboptimal solutions \citep{varando2020learning, lippe2021efficient, zhang2023boosting}. \textbf{Granger-based methods}, inspired by \citep{granger1969investigating,granger2015spectral}, offer an intriguing perspective on causal discovery. These methods utilize auto-regression algorithms under the assumption of additive noise to assess if one time series can predict another, thereby identifying causal relationships. However, they tend to exhibit lower precision when working with limited observed time steps.

To overcome the limitations of existing approaches, such as low sample efficiency in constraint-based methods, suboptimal solutions from acyclicity regularizers in score-based methods and low precision when limited observed time steps in Granger-based methods, we propose a novel \textbf{S}hort-\textbf{T}erm \textbf{I}nvariance-based \textbf{C}onvolutional causal discovery approach (\textbf{STIC}). STIC leverages the properties of short-term invariance to enhance the sample efficiency and accuracy of causal discovery.
More concretely, by sliding a window along the entire time-series data, STIC constructs batches of window observations that possess invariant characteristics and improves sample utilization. Unlike existing score-based methods, our model does not rely on predefined acyclicity constraints to avoid local optimization. As the window observations move along the temporal chain, the structure of the window causal graph exhibits periodic patterns, demonstrating short-term time invariance. Simultaneously, the conditional probabilities of causal effects between variables remain unchanged as the window observations slide, indicating short-term mechanism invariance. 
The contributions of our work can be summarized as follows:

\begin{itemize}
  \item We propose \textbf{STIC}, the \textbf{S}hort-\textbf{T}erm \textbf{I}nvariance-based \textbf{C}onvolutional causal discovery approach, which leverages the properties of short-term invariance to enhance the sample efficiency and accuracy of causal discovery.
  \item  STIC uses the time-invariance block to capture the causal relationships among variables, while employing the mechanism-invariance block for the transform function.
  \item  To dynamically capture the contemporaneous and time-lagged causal structures of the observed variables, we establish the equivalence between the convolution of the space-domain (contemporaneous) and time-domain (time-lagged) components, and the multivariate Fourier transform (the underlying generative mechanism) of time-series data.
  \item We conduct experiments to evaluate the performance of STIC on synthetic and benchmark datasets. The experimental results show that STIC achieves the state-of-the-art results on synthetic time-series datasets, even when dealing with relatively limited observed time steps. Experiments demonstrate that our approach outperforms baseline methods in causal discovery from time-series data.
\end{itemize}

\section{Background \label{Background}}

In this section, we introduce the background of causal discovery from time-series data. Firstly, we show all symbols and their definitions in Section \ref{Symbol_Summarization}. Secondy, in Section \ref{Problem_Definition}, we present the problem definition and formal representation of window causal graph. Thirdly, in Section \ref{Causal_Invariance}, we introduce the concepts of short-term time invariance and mechanism invariance. Building upon these concepts, we derive an independence property specific to window causal graph. Fourthly, in Section \ref{Necessity_of_convolution}, we delve into the theoretical aspects of our approach. Specifically, we establish the equivalence between the convolution operation and the underlying generative mechanism of the observed time-series data. This theoretical grounding provides a solid basis for the proposed STIC approach. Finally, in Section \ref{Granger_Causality}, we introduce Granger causality, an auto-regressive approach to causal discovery from time-series data.

\begin{figure}[!t]
  \centering
  \includegraphics[width=0.8\linewidth]{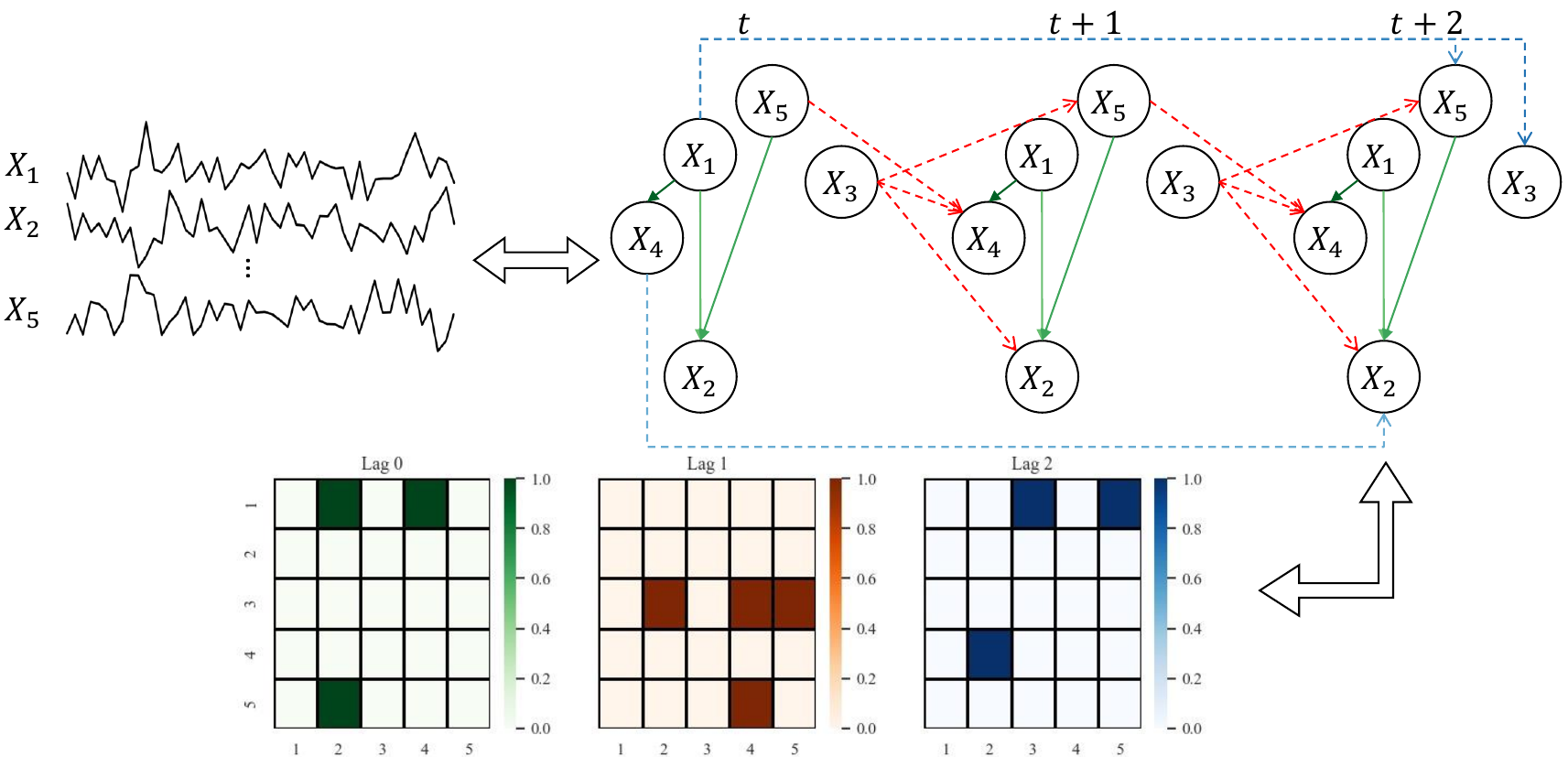}
  \caption{A example showing the correspondence among the given observed variables, the underlying window causal graph, and the window causal matrix. The observed dataset consists of $d=5$ observed variables, and the true maximum lag $\overline{\tau}$ is $2$. In $\mathcal{W}\in \mathbb{R}^{5\times 5\times 3}$, for each $\mathcal{W}_{i,j}^{\tau}$ represents the causal effect of $X_i$ on $X_j$ with $\tau$ time lags. For example, the blue lines in window causal graph indicate the following three causal effects with time lags $\tau=2$ at any time step $t$, i.e. $\mathcal{W}_{1,3}^{2}=1\Rightarrow X_{1}\stackrel{2}{\longrightarrow} X_{3};\mathcal{W}_{1,5}^{2}=1\Rightarrow X_{1}\stackrel{2}{\longrightarrow} X_{5}; \mathcal{W}_{4,2}^{2}=1\Rightarrow X_{4}\stackrel{2}{\longrightarrow} X_{2}$. Moreover, the red lines indicate the causal relationships with time lags $\tau=1$, i.e. $\mathcal{W}_{3,2}^{1}=1\Rightarrow X_{3}\stackrel{1}{\longrightarrow} X_{2}; \mathcal{W}_{3,4}^{1}=1\Rightarrow X_{3}\stackrel{1}{\longrightarrow} X_{4};\mathcal{W}_{3,5}^{1}=1\Rightarrow X_{3}\stackrel{1}{\longrightarrow} X_{5}; \mathcal{W}_{5,4}^{1}=1\Rightarrow X_{5}\stackrel{1}{\longrightarrow} X_{4}$. Finally, the green lines represent contemporaneous causal relationships, i.e. $\mathcal{W}_{1,2}^{0}=1\Rightarrow X_{1}\stackrel{0}{\longrightarrow} X_{2}; \mathcal{W}_{1,4}^{0}=1\Rightarrow X_{1}\stackrel{0}{\longrightarrow} X_{4}; \mathcal{W}_{5,2}^{0}=1\Rightarrow X_{5}\stackrel{0}{\longrightarrow} X_{2}$.}
  \label{fig1}
\end{figure}

\subsection{Symbol Summarization\label{Symbol_Summarization}}

Firstly, to better represent the symbols used in Section \ref{Background}, we arrange a table to summarize and show their definitions, as shown in Table \ref{symbols_table}.

\begin{table}[!t]
  \centering
  \caption{Summary of symbol definitions in Section \ref{Background}.}
  \label{symbols_table}
  \resizebox{0.8\linewidth}{!}{%
  \begin{tabular}{c|c}
    \toprule[1pt]
  Symbol & Description \\ \midrule
  $d$ & The number of observed variables \\ \midrule
  $T$ & The length of observed time steps \\ \midrule
  $X_i^t$ & The observed value of the $i$-th variable at the $t$-th time step \\ \midrule
  $X_i=\{X_i^1, \cdots, X_i^T\}\in{\mathbb{R}^{T}}$ & The observed value of $i$-th variable within all $T$ time steps \\ \midrule
  $\mathcal{X}=\{X_1, \cdots, X_d\}\in{\mathbb{R}^{d\times T}}$ & The observed dataset \\ \midrule
  $\widetilde{\tau}$ & The maximum time lag \\ \midrule
  $\mathcal{G}$ & The underlying window causal graph \\ \midrule
  $\mathcal{V}=\{X_1,...,X_d\}$  & The nodes within the graph $\mathcal{G}$ \\ \midrule
  $\mathcal{E}$ & The contemporaneous and time-lagged relationships among nodes $\mathcal{V}$ \\ \midrule
  $\mathcal{W}\in \mathbb{R}^{d\times d\times (\widetilde{\tau}+1)}$ & The window causal matrix \\ \midrule
  $X_{i}\stackrel{\tau}{\longrightarrow} X_{j}$ & The causal relationship with $\tau$ lags between $X_i$ and $X_j$\\ \midrule
  $Pa_{t}^{\tau}(\cdot)$ & The set of parents of a variable with $\tau$ time lags at time step $t$\\ \midrule
  $Pa_{t}^{\cdot}(\cdot)$ & The set of parents of a variable with all time lags range from 0 to $\widetilde{\tau}$ at time step $t$\\ \midrule
  $\underset{t}{\perp \!\!\! \perp}^{\tau}$ & The conditional independence with $\tau$ time lags at time step $t$ \\ \midrule
  $Pa_{\mathcal{G}}(\mathcal{X})$  & The relationships among $\mathcal{X}$ in the window causal graph $\mathcal{G}$\\ \midrule
  $E$ & The noise term\\ \midrule
  $f$ &  The underlying functions among $\mathcal{X}$\\ \midrule
  $\mathcal{F}(\mathcal{X})$ & The multivariate Fourier transform of $\mathcal{X}$ \\ \midrule
  $\omega$ & The angular frequency\\ \midrule
  $\hat{f},h,g$ & The functions in intermediate processes\\ \midrule
  $*$ & The convolution operation\\ \midrule
  $\propto$ & The directly proportional relationship\\ \midrule
  $\sigma_{\tau}^2(X_i|\mathcal{X})$ & The variance of predicting $X_i$ using $\mathcal{X}$ with $\tau$ time lags\\ 
   \bottomrule[1pt]
  \end{tabular}%
  }
\end{table}

\subsection{Problem Definition\label{Problem_Definition}}

Let an observed dataset denoted as $\mathcal{X}=\{X_1, \cdots, X_d\}\in{\mathbb{R}^{d\times T}}$, which consists of $d$ observed continuous time-series variables. Each variable $X_i$ is represented as a time sequence $X_i=\{X_i^1, \cdots, X_i^T\}$ with the length of $T$. Here, each $X_i^t$ corresponds to the observed value of the $i$-th variable $X_i$ at the $t$-th time step. Unlike graph embedding algorithms \citep{cheng2020time2graph,cheng2021time2graph+} which aims to learn time series representations, the objective of causal discovery is to uncover the underlying structure within time-series data, which represents boolean relationships between observed variables. Furthermore, following the Consistency Throughout Time assumption \citep{spirtes2000causation, zhang2002strong, robins2003uniform, kalisch2007estimating, entner2010causal, assaad2022survey}, the objective of causal discovery from time-series data is to uncover the underlying window causal graph $\mathcal{G}$ as an invariant causal structure. The true window causal graph for $\mathcal{X}$ encompasses both intra-slice causality with 0 time lags and inter-slice causality with time lags ranging from $1$ to $\widetilde{\tau}$. Here, $\widetilde{\tau}$ denotes the maximum time lag. Mathematically, the window causal graph is defined as a finite Directed Acyclic Graph (DAG) denoted by $\mathcal{G}=(\mathcal{V},\mathcal{E})$. The set $\mathcal{V}=\{X_1,...,X_d\}$ represents the nodes within the graph $\mathcal{G}$, wherein each node corresponds to an observed variable $X_i$. The set $\mathcal{E}$ represents the contemporaneous and time-lagged relationships among these nodes, encompassing all $2^{(\widetilde{\tau}+1)\times d}$ possible combinations. The window causal graph is often represented by the window causal matrix, which is defined as follows.

\begin{definition}[Window Causal Matrix\label{window_causal_Matrix}]
  The window causal graph $\mathcal{G}$, which captures both contemporaneous and time-lagged causality, can be effectively represented using a three-dimensional boolean matrix $\mathcal{W}\in \mathbb{R}^{d\times d\times (\widetilde{\tau}+1)}$. Each entry $\mathcal{W}_{i,j}^{\tau}$ in the boolean matrix corresponds to the causal relationship between variables $X_i$ and $X_j$ with $\tau$ time lags. To be more specific, if $\mathcal{W}_{i,j}^{\tau=0}= 1$, it signifies the presence of an intra-slice causal relationship between $X_i$ and $X_j$, meaning they influence each other at the same time step. On the other hand, if $\mathcal{W}_{i,j}^{\tau>0}= 1$, it indicates that $X_i$ causally affects $X_j$ with $\tau$ time lags.
\end{definition}

Figure \ref{fig1} provides a visual example of a window causal graph along with its corresponding matrix defined in Definition \ref{window_causal_Matrix}. As shown in Figure \ref{fig1}, the time-series causal relationships of the form $X_{i}\stackrel{\tau}{\longrightarrow} X_{j}$ can be represented as $\mathcal{W}_{i,j}^{\tau}=1$. Conversely, $\mathcal{W}_{i,j}^{\tau}=1$ in the boolean matrix indicates that the value $X_{i}^{t}$ at any time step $t$ influences the value $X_{j}^{t+\tau}$ with $\tau$ time lags later. 

\subsection{Short-Term Causal Invariance\label{Causal_Invariance}}

There has been an assertion that causal relationships typically exhibit short-term time and mechanism invariance across extensive time scales \citep{entner2010causal,liu2023causal,zhang2017causal}. These two aspects of invariance are commonly regarded as fundamental assumptions of causal invariance in causal discovery from time-series data. In the following, we will present the definitions for these two forms of invariance.

\begin{definition}[Short-Term Time Invariance\label{time_invariance}]
  Given $\mathcal{X}\in {\mathbb{R}^{d\times T}}$, for any $X_{i},X_{j},\tau\ge 0$, if $X_{i}\in Pa_{t}^{\tau}(X_{j})$ at time $t$, then there exists $X_{i}\in Pa_{t'}^{\tau}(X_{j})$ at time $t'\neq t$ in a short period of time, where $Pa_{t}^{\tau}(\cdot)$ denotes the set of parents of a variable with $\tau$ time lags at time step $t$.
\end{definition}

Short-term time invariance refers to the stability of parent-child relationships over time. In other words, it implies that the dependencies between variables remain consistent regardless of specific time points. For instance, considering Figure \ref{fig1}: $X_{5}$ is a parent of $X_{4}$ with time lag $\tau=1$ at $t$, then $X_{5}$ will also be a parent of $X_{4}$ with time lag $\tau=1$ at $t'=t+1$; similarly, when $\tau=0$, if $X_{5}$ is a parent of $X_{2}$ at $t$, then $X_{5}$ will be a parent of $X_{2}$ at no matter $t'=t+1$ or $t'=t+2$.

\begin{definition}[Short-Term Mechanism Invariance\label{mechanism_invariance}]
  For any $X_{i}$, the conditional probability distribution $P(X_{i}|Pa_{t}^{\cdot}(X_{i}))$ remains constant across the short-term temporal chain. In other words, for any time step $t$ and $t'$, it holds that $P(X_{i}|Pa_{t}^{\cdot}(X_{i}))=P(X_{i}|Pa_{t'}^{\cdot}(X_{i}))$, where $Pa_{t}^{\cdot}(X_i)$ means the set of parents of $X_i$ with all time lags range from 0 to $\widetilde{\tau}$ at time step $t$.
\end{definition}

In particular, based on Definition \ref{mechanism_invariance}, short-term mechanism invariance implies that conditional probability distributions remain constant over time. For instance, in Figure \ref{fig1}, we have $Pa_{t}^{\cdot}(X_{2})=\{X_{3},X_{1},X_{5}\}=Pa_{t+1}^{\cdot}(X_{2})$. Then, we have $P(X_{2}|Pa_{t}^{\cdot}(X_{2}))=P(X_{2}|Pa_{t+1}^{\cdot}(X_{2}))$

Building upon the definitions of short-term time invariance and mechanism invariance, we can derive the following lemma, which characterizes the invariant nature of independence among variables. Inspired by causal invariance \citep{entner2010causal}, we further provide a detailed proof procedure as outlined below.

\begin{lemma}[Independence Property\label{Independence_Invariance}]
  Given $\mathcal{X}\in{\mathbb{R}^{d\times T}}$ be the observed dataset. 
  If we have $X_i\underset{t}{\perp \!\!\! \perp}^{\tau} X_j | X_k,...,X_l$, then we have $X_i\underset{t'}{\perp \!\!\! \perp}^{\tau} X_j | X_k,...,X_l$. $\underset{t}{\perp \!\!\! \perp}^{\tau}$ means conditional independence with $\tau$ time lags at time step $t$.
\end{lemma}

\begin{proof}
  Due to the short-term time invariance of the relationships among variables and the short-term mechanism invariance of conditional probabilities, different value $X_i^t$ and $X_i^{t'}$ of $X_i$ is mapped to the same variable $X_i$ in the window causal graph $\mathcal{G}$. Consequently, $Pa_{t}^{\tau}(X_i)$ and $Pa_{t'}^{\tau}(X_i)$ correspond to the same variable set. Thus, if the condition $X_i\underset{t}{\perp \!\!\! \perp}^{\tau} X_j | X_k,...,X_l$ holds, then $X_i\underset{\mathcal{G}}{\perp \!\!\! \perp}^{\tau} X_j | X_k,...,X_l$ holds in the window causal graph $\mathcal{G}$, which further implies $X_i\underset{t'}{\perp \!\!\! \perp}^{\tau} X_j | X_k,...,X_l$.
\end{proof}

This lemma establishes that, in an identifiable window causal graph, the independence property remains invariant with time translation. Leveraging this insight, we can transform the observed time series into window observations to perform causal discovery while maintaining the invariance conditions, as outlined in Section \ref{Window_Representation}.

\subsection{Necessity of Convolution\label{Necessity_of_convolution}}

Granger demonstrated, through the Cramer representation and the spectral representation of the covariance sequence \citep{granger1969investigating,mills2013granger,granger2015spectral}, that time-series data can be decomposed into a sum of uncorrelated components. Inspired by these representations and the concept of graph Fourier transform \citep{shuman2013emerging,sandryhaila2013discrete,sardellitti2017graph}, we propose considering a underlying function $\mathcal{X}=f(Pa_{\mathcal{G}}(\mathcal{X}),\mathcal{W})+E$, where $Pa_{\mathcal{G}}(\mathcal{X})$ denotes relationships among $\mathcal{X}$ in the window causal graph $\mathcal{G}$ and $E$ is the noise term, to describe the generative process of the observed dataset $\mathcal{X}=\{X_1, \cdots, X_d\}\in{\mathbb{R}^{d\times T}}$, with an underlying window causal matrix $\mathcal{W}\in \mathbb{R}^{d\times d\times (\widetilde{\tau}+1)}$. We can then decompose $f(Pa_{\mathcal{G}}(\mathcal{X}),\mathcal{W})$ into Fourier integral forms:

\begin{equation}
  \begin{aligned}
    \mathcal{X} & = f(Pa_{\mathcal{G}}(\mathcal{X}),\mathcal{W})+E             \\
      & = \hat{f}(s,t)+E
  \end{aligned}
  \label{E2_0}
\end{equation}

\begin{figure*}[ht]
  \centering
  \includegraphics[width=1\linewidth]{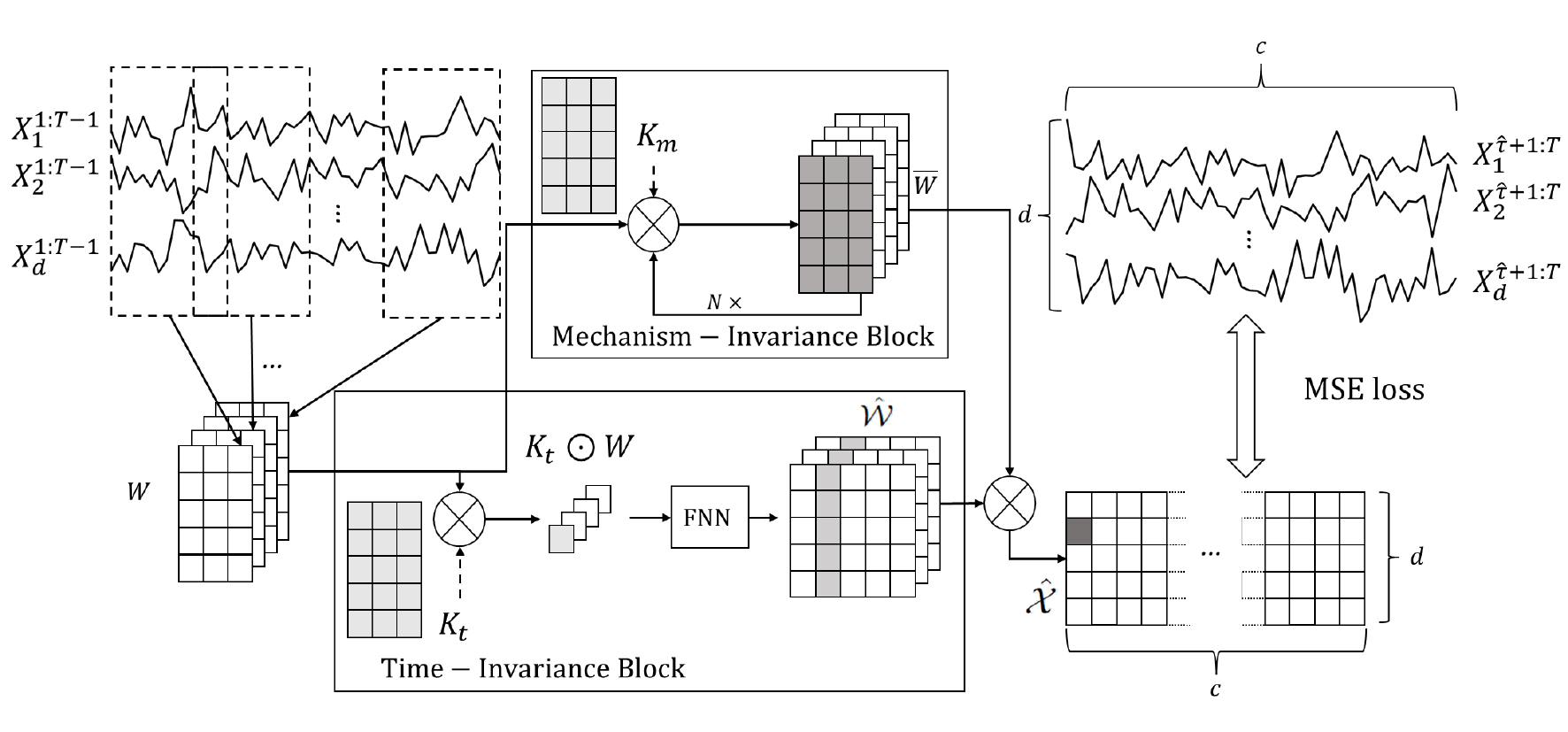}
  \caption{An illustration of the STIC framework. Let $\mathcal{X}=\{X_1, \cdots, X_d\}\in\mathbb{R}^{d\times T}$ be the observed dataset, representing $d$ observed continuous time series of the same length $T$. First, we convert the observations of the first $T-1$ time steps, $\mathcal{X}^{1:T-1}=\{X_1^{1:T-1}, \cdots, X_d^{1:T-1}\}\in\mathbb{R}^{d\times (T-1)}$, into a window representation $W\in\mathbb{R}^{d\times\hat{\tau}\times c}$ using a sliding window with a predefined window length $\hat{\tau}$ and step length 1, where $c=T-\hat{\tau}$.
  \textbf{Time-Invariance Block} ($B_t$): In order to better discover the causal structure from $\mathcal{X}$, we use convolution kernel $K_t \in \mathbb{R}^{d\times\hat{\tau}}$ to act on $W$, and get the common representation $K_t\odot W_\psi$ of $\mathcal{X}$ for each window observations $W_\psi$. Afterwards, we pass the commonality through an FNN network to obtain a predicted window causal matrix $\hat{\mathcal{W}}\in\mathbb{R}^{d\times d\times\hat{\tau}}$. 
  \textbf{Mechanism-Invariance Block} ($B_m$): To identify numerical transform in window causal graph, we use another convolution kernel $K_m\in\mathbb{R}^{d\times\hat{\tau}}$ in each $B_m$ to transform $W$. Then we output $\overline{W}\in\mathbb{R}^{d\times\hat{\tau}\times c}$ as the prediction of $f(W)$. 
  Next, we do hadamard product of each $\overline{W}_{\psi}^\tau \in\mathbb{R}^{d}$ in $\overline{W}$ and each $\hat{\mathcal{W}}^\tau \in\mathbb{R}^{d\times d}$ in $\hat{\mathcal{W}}$ to get the predicted $\hat{\mathcal{X}}^{\hat{\tau}+\psi}$ until we get all $\hat{\mathcal{X}}\in\mathbb{R}^{d\times c}$. Finally, we calculate the Mean Squared Error (MSE) loss between $\hat{\mathcal{X}}$ and $\mathcal{X}$, and adopt gradient descent to optimize the parameters within the time-invariance and mechanism-invariance blocks.
  }
  \label{fig2}
\end{figure*}

Here, $s$ and $t$ denote the spatial and temporal projections, respectively, of $f(Pa_{\mathcal{G}}(\mathcal{X}),\mathcal{W})$. Equation \ref{E2_0} is derived from the observation that the contemporaneous part in time-series data corresponds to the spatial domain, while the time-lagged part corresponds to the temporal domain. Therefore, we employ the multivariate Fourier transform,

\begin{equation}
  \begin{aligned}
    \mathcal{F}(\mathcal{X}) & = \iint_{-\infty}^{\infty} \hat{f}(x,y;s,t)e^{-i\omega (sx+ty)}dxdy                                 \\
                   & \propto \iint_{-\infty}^{\infty} h(\hat{s})g(\hat{t})e^{-i\omega (\hat{s}+\hat{t})}d\hat{s}d\hat{t}
  \end{aligned}
  \label{E2_1}
\end{equation}

where $\hat{s}$ represents the spatial domain component, $\hat{t}$ represents the temporal domain component, and $\omega$ represents the angular frequency along with transform function $\hat{f},h$ and $g$. The first line corresponds to applying the Fourier transform to both sides of Equation \ref{E2_0}. In the second line, inspired by the Time-Independent Schrödinger Equation \citep{zabusky1968solitons,rana2019time}, we assume that $f(x,y;s,t)$ can be decomposed into the spatial and temporal domains, i.e., $\hat{f}(x,y;s,t) = h(\hat{s})g(\hat{t})$. Next, by utilizing the convolution theorem \citep{zayed1998convolution} for tempered distributions, which states that under suitable conditions the Fourier transform of a convolution of two functions (or signals) is the pointwise product of their Fourier transform, i.e., $\mathcal{F}{(h*g)} = \mathcal{F}{(h)} \cdot \mathcal{F}{(g)}$, where $\mathcal{F}{(\cdot)}$ represents the Fourier transform, we convert the convolution formula into the following expression:

\begin{equation}
  \begin{aligned}
    \mathcal{F}[h(\hat{s})*g(\hat{t})] & \propto \mathcal{F}(h(\hat{s})) \cdot \mathcal{F}(g(\hat{t}))                                                                      \\
                                       & \propto \int_{-\infty}^{\infty} h(\hat{s})e^{-i\omega\hat{s}}d\hat{s}\int_{-\infty}^{\infty} g(\hat{t})e^{-i\omega\hat{t}}d\hat{t} \\
                                       & \propto \mathcal{F}(\mathcal{X})
  \end{aligned}
  \label{E2_2}
\end{equation}

The first line of the Formula \ref{E2_2} is obtained through the convolution theorem, while the second line expands $\mathcal{F}(h(\hat{s}))$ and $\mathcal{F}(g(\hat{t}))$ using the Fourier transform. The third line is derived from Equation \ref{E2_1}. Therefore, it indicates that the observed dataset $\mathcal{X}$ can be obtained by convolving the convolution kernel with temporal information and the spatial details, which we will deal with corresponding to the two kinds of invariance. We posit that the convolution operation precisely aligns with the functional causal data generation mechanism, i.e., $\mathcal{X}\propto h(\hat{s})*g(\hat{t})$. Conversely, the convolution operation can be used to analytically model the generation mechanism of functional time-series data. Therefore, we will employ the convolution operation to extract the functional causal relationships within the window causal graph. In conclusion, the equivalence between the generation mechanism of time-series causal data and convolution operations serves as motivation to incorporate convolution operations into our STIC framework.

\subsection{Granger Causality\label{Granger_Causality}}

Granger causality \citep{granger1969investigating,pavasant2021spatio,assaad2022survey} is a method that utilizes numerical calculations to assess causality by measuring fitting loss and variance. Formally, we say that a variable $X_i$ Granger-causes another variable $X_j$ when the past values of $X_i$ at time $t$ (i.e., $X_i^1, \cdots, X_i^{t-1}$) enhance the prediction of $X_j$ at time $t$ (i.e., $X_j^t$) compared to considering only the past values of $X_j$. The definition of Granger causality is as follows:

\begin{definition}[Granger Causality\label{Granger_Causality_defin}]
  Let $\mathcal{X}=\{X_1, \cdots, X_d\}\in{\mathbb{R}^{d\times T}}$ be a observed dataset containing $d$ variables. If $\sigma_{\tau}^2(X_j|\mathcal{X})<\sigma_{\tau}^2(X_j|\mathcal{X}-X_i)$, where $\sigma_{\tau}^2(X_j|\mathcal{X})$ denotes the variance of predicting $X_j$ using $\mathcal{X}$ with $\tau$ time lags, we say that $X_i$ causes $X_j$, which is represented by $\mathcal{W}_{i,j}^{\tau}=1$.
\end{definition}

In simpler terms, Granger causality states that $X_i$ Granger-causes $X_j$ if past values of $X_i$ (i.e., $X_i^{t'}$) provide unique and statistically significant information for predicting future values of $X_j$ (i.e., $X_j^t$). Therefore, following the definition of Granger causality, we can approach causal discovery as an autoregressive problem.

\section{Method\label{method}}

In this section, we introduce STIC, which involves four components: Window Representation, Time-Invariance Block, Mechanism-Invariance Block, and Parallel Blocks for Joint Training. The process is depicted in Figure \ref{fig2}. Firstly, we transform the observed time series into a window representation format, leveraging Lemma \ref{Independence_Invariance}. Next, we input the window representation into both the time-invariance block and the mechanism-invariance block ($B_t$ and $B_m$ in Figure \ref{fig2}). Finally, we conduct joint training using the extracted features from two kinds of parallel blocks. In particular, the time-invariance block $B_t$ generates the estimated window causal matrix $\hat{\mathcal{W}}$. To better represent the symbols used in Section \ref{method}, we also arrange a table to summarize and show their definitions, as shown in Table \ref{symbols_table_method}. The subsequent subsections provide a detailed explanation of the key components of STIC. 

\begin{table}[!t]
  \centering
  \caption{Summary of symbol definitions in Section \ref{method}.}
  \label{symbols_table_method}
  \resizebox{0.8\linewidth}{!}{%
  \begin{tabular}{c|c}
    \toprule[1pt]
  Symbol & Description \\ \midrule
  $\overline{\tau}$  & The predefined maximum time lag \\ \midrule
  $\hat{\tau}$ & The predefined window length, $\hat{\tau}=\overline{\tau}+1$\\ \midrule
  $W\in\mathbb{R}^{d\times\hat{\tau}\times c}$ & The window representation, where $c=T-\hat{\tau}$\\ \midrule
  $B_t$ & The time-invariance block\\ \midrule
  $K_t \in \mathbb{R}^{d\times\hat{\tau}}$ & The convolution kernel in the time-invariance block\\ \midrule
  $\odot$ & The Hadamard product \\ \midrule
  $\hat{\mathcal{W}}\in\mathbb{R}^{d\times d\times\hat{\tau}}$ & The predicted window causal matrix \\ \midrule
  $B_m$ & The mechanism-invariance block \\ \midrule
  $K_m\in\mathbb{R}^{d\times\hat{\tau}}$ & The convolution kernel in the mechanism-invariance block\\ \midrule
  $\overline{W}\in\mathbb{R}^{d\times\hat{\tau}\times c}$ & The prediction of $f(W)$\\ \midrule
  $\hat{\mathcal{X}}\in\mathbb{R}^{d\times c}$ & The prediction of the observed dataset\\ \midrule
  $f_1:\mathbb{R}^{c\times d\times \hat{\tau}}\rightarrow\mathbb{R}^{d\times d\times \hat{\tau}}$ & The feed-forward neural network\\ \midrule
  $\hat{\mathcal{W}}_{i,j}^\tau$  & The estimated binary existence of the causal effect of $X_i$ on $X_j$ with $\tau$ time lags\\ \midrule
  $p$ & The threshold used to eliminate edges with low probability of existence \\ \midrule
  $f_2:\mathbb{R}^{d\times \hat{\tau}}\rightarrow\mathbb{R}^{d\times \hat{\tau}}$  & The estimated transformation function\\ 
   \bottomrule[1pt]
  \end{tabular}%
  }
\end{table}

\subsection{Window Representation\label{Window_Representation}}

The observed dataset $\mathcal{X}\in\mathbb{R}^{d\times T}$ contains $d$ observed continuous time series (variables) with $T$ time steps. We also define a predefined maximum time lag as $\overline{\tau}$. To ensure that the entire causal contemporaneous and time-lagged influence is observed, we calculate the minimum length of the window that can capture this influence as $\hat{\tau}=\overline{\tau}+1$. To construct the window observations, we select the observed values from the first $T-1$ time steps, i.e. $\mathcal{X}^{1:T-1}=\{X_1^{1:T-1}, \cdots, X_d^{1:T-1}\}\in\mathbb{R}^{d\times (T-1)}$. Using a sliding window approach along the temporal chain of observations, we create window observations of length $\hat{\tau}$ and width $d$, with a step size of 1. This process results in $c=T-\hat{\tau}$ window observations $W_\psi$ where $\psi=1,...,c$. These window observations are referred to as the window representation $W$, as illustrated in Figure \ref{fig_sliding}.

\begin{figure}[!t]
  \centering
  \includegraphics[width=0.6\linewidth]{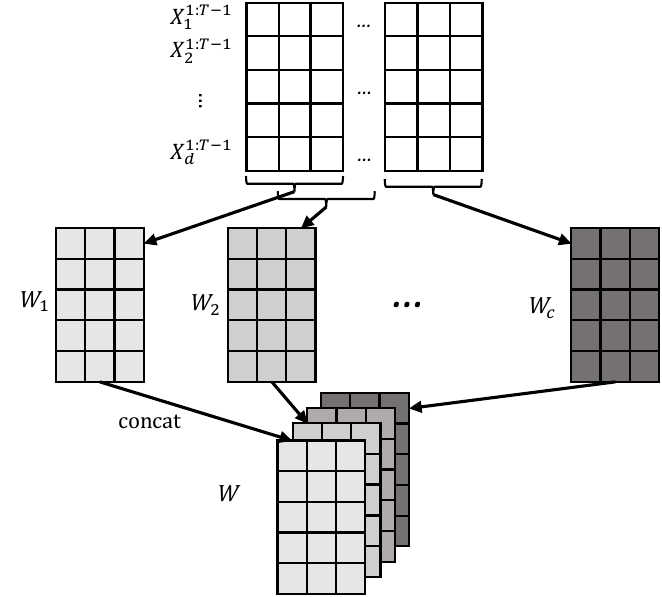}
  \caption{Window representation. First, we get $c$ matrices $W_\psi$ by sliding window with predefined window length $\hat{\tau}$ and step size $1$, where each $W_\psi\in\mathbb{R}^{d\times \hat{\tau}}, \psi=1,...,c$ represents the data we observe in the window. Then, we concatenate the obtained $W_\psi$ together to get the final window representation $W\in\mathbb{R}^{d\times \hat{\tau}\times c}$.}
  \label{fig_sliding}
\end{figure}

\subsection{Time-Invariance Block\label{Time_Invariance_Block}}

According to Definition \ref{time_invariance}, the causal relationships among variables remain unchanged as time progresses. Exploiting this property, we can extract shared information from the window representation $W$ and utilize it to finally obtain the estimated window causal matrix $\hat{\mathcal{W}}$. Inspired by convolutional neural networks used in causal discovery\citep{nauta2019causal}, we introduce a invariance-based convolutional network structure denoted as $B_t$ to incorporate temporal information within the window representation $W$. For each window observation $W_\psi\in \mathbb{R}^{d\times\hat{\tau}}$, we employ the following formula to aggregate similar information among the time series within the window observations

\begin{equation}
  \hat{\mathcal{W}}=f_1(K_t \odot W_1,...,K_t \odot W_c)
  \label{E3}
\end{equation}

Here, shared $K_t\in\mathbb{R}^{d\times\hat{\tau}}$ represents a learnable extraction kernel utilized to extract information from each window observation. The symbol $\odot$ denotes the Hadamard product between matrices, and $f_1$ refers to a neural network structure. By applying the Hadamard product with the shared kernel $K_t$, the resulting output exhibits similar characteristics across the time series. Moreover, $K_t$ serves as a time-invariant feature extractor, capturing recurring patterns that appear in the input series and aiding in forecasting short-term future values of the target variable. In Granger causality, these learned patterns reflect causal relationships between time series, which are essential for causal discovery \citep{nauta2018temporal}. To ensure the generality of STIC, we employ a simple feed-forward neural network (FNN) $f_1:\mathbb{R}^{c\times d\times \hat{\tau}}\rightarrow\mathbb{R}^{d\times d\times \hat{\tau}}$ to extract shared information from each $K_t \odot W_\psi, \psi=1,...,c$. Furthermore, we impose a constraint to prohibit self-loops in the estimated window causal matrix $\hat{\mathcal{W}}$ when the time lag is zero. That is:

\begin{eqnarray}
  \hat{\mathcal{W}}_{i,j}^\tau=\left\{
  \begin{array}{rcl}
    0        &  & {\text{if }i=j\text{ and }\tau=0} \\
    0        &  & {\text{if }\hat{\mathcal{W}}_{i,j}^\tau < p}       \\
    1 &  & {\text{else}}
  \end{array} \right.,
  \label{E4}
\end{eqnarray}

where $\hat{\mathcal{W}}_{i,j}^\tau$ represents the estimated binary existence of the causal effect of $X_i$ on $X_j$ with a time delay of $\tau\in\{0,...,\overline{\tau}\}$, and $p$ is a threshold used to eliminate edges with low probability of existence.

\subsection{Mechanism-Invariance Block\label{Mechanism_Invariance_Block}}

As stated in Definition \ref{mechanism_invariance}, the causal conditional probability relationships among the time series remain unchanged as time varies. Consequently, the causal functions between variables also remain constant over time. With this in mind, our objective in $B_m$ is to find a unified transform function $f_2:\mathbb{R}^{d\times \hat{\tau}}\rightarrow\mathbb{R}^{d\times \hat{\tau}}$ that accommodates all window observations.
To achieve this goal, as depicted in Figure \ref{fig2}, we employ a convolution kernel $K_m\in\mathbb{R}^{d\times\hat{\tau}}$ as $f_2$. This kernel performs a Hadamard product operation with each window $W_\psi\in\mathbb{R}^{d\times\hat{\tau}}$ in $W$, where $\psi=1,...,c$. Subsequently, we employ the Parametric Rectified Linear Unit (PReLU) activation function \citep{zhu2017diverse} to obtain the output $\overline{W}_{\psi}\in\mathbb{R}^{d\times\hat{\tau}}$, 

\begin{equation}
  \overline{W}_{\psi}=PReLU(K_m\odot W_{\psi})
  \label{E2}
\end{equation}

Each $\overline{W}_{\psi}$ represents the transformed matrix obtained from the window observation $W_\psi$ by a unified transform function $f_2$ implemented with convolution kernel $K_m$. Each $\overline{W}_{\psi}$ is finally used to predict $\hat{\mathcal{X}}^{\hat{\tau}+\psi}$. Note that this transform function $f_2$ can also be composed of $N$ different but equal dimensional kernels $K_m^1,...,K_m^N \in\mathbb{R}^{d\times\hat{\tau }}$, which are nested to perform complex nonlinear transformations. After $f_2$, the value inside the window $W_\psi\in\mathbb{R}^{d\times\hat{\tau}}$ is then pressed for $\hat{\mathcal{W}}$-selected column summation to predict $\hat{\mathcal{X}}^{\hat{\tau}+\psi}\in\mathbb{R}^{d}$.

\begin{figure*}[!t]
  \centering
  \subfigure{\includegraphics[width=0.49\linewidth]{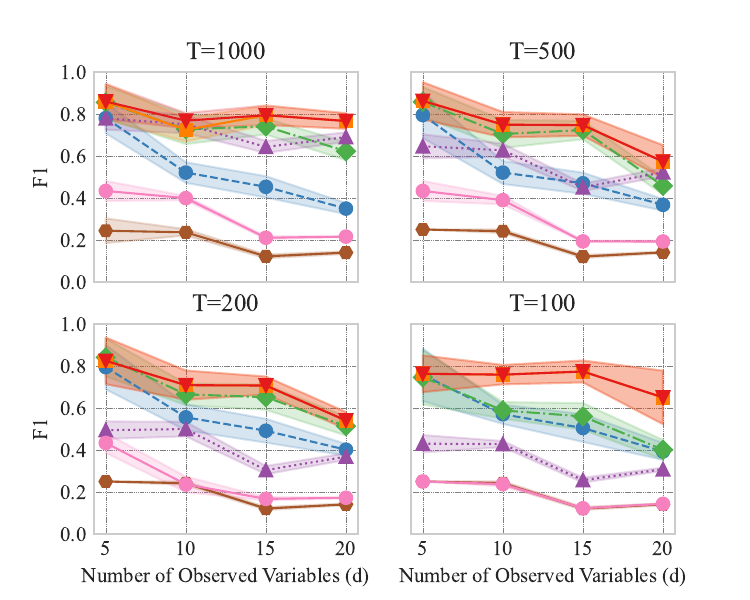}%
  \label{fig3:a}}
  \subfigure{\includegraphics[width=0.49\linewidth]{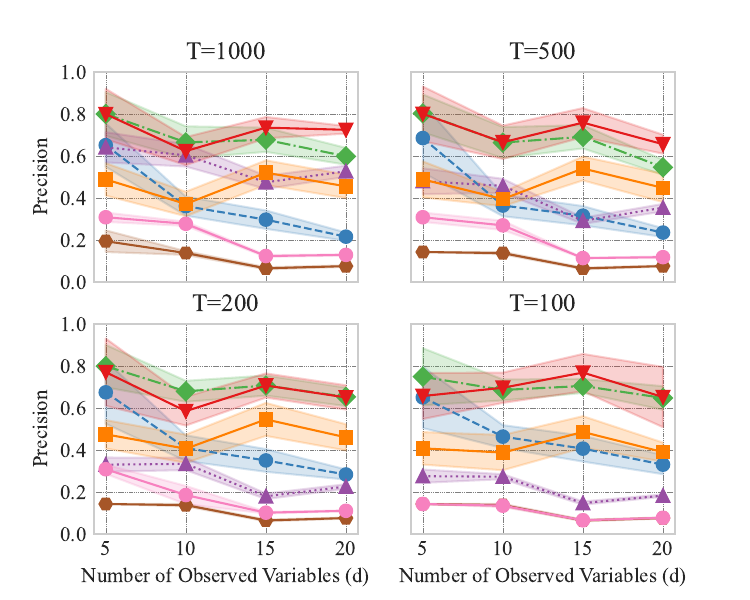}%
  \label{fig3:b}}\vspace{-5mm}

  \subfigure{\includegraphics[width=\linewidth]{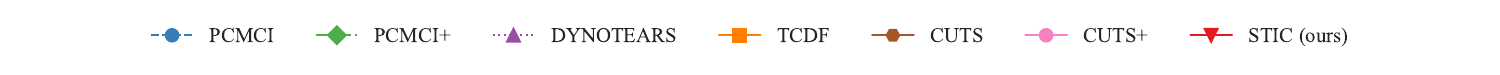}%
  \label{fig3:c}}
  \caption{The results of F1 (detailed in Figure \ref{fig3} left) and precision (detailed in Figure \ref{fig3} right) evaluated on linear Gaussian datasets with varying numbers of variables ($d$) and observed time steps ($T$). The observed data $\mathcal{X}$ is generated by sampling $d$ time series with $T$ observed time steps from a linear Gaussian distribution. We consider different values of $d$ ranging from 5 to 20, and varying observed time steps $T$ including 100, 200, 500, and 1000.}
  \label{fig3}
\end{figure*}

\subsection{Parallel Blocks for Joint Training\label{Joint_Training}}

So far, we have obtained the estimated window causal matrix $\hat{\mathcal{W}}$ by using $B_t$. In addition, we also obtained the transformed matrix $\overline{W}$ with $B_m$. We used convolutional neural networks in both $B_t$ and $B_m$. Their structures are similar, but their functions and purposes are different. In $B_t$, we focus on the shared underlying unified structure of all window observations. Following the Definition \ref{time_invariance} of short-term time invariance, we choose a convolutional neural network structure with translation invariance \citep{kayhan2020translation,singh2023orthogonal}. We expect that $f_1$ with $K_t$ as the main component can extract the invariant structure of the window representation $W$. In $B_m$, we focus on the convolution kernel $K_m$, which is expected to serve as a unified transform function $f_2$ to satisfy the Definition \ref{mechanism_invariance} of short-term mechanism invariance and perform complex nonlinear transformations.

Based on Definition \ref{Granger_Causality_defin} described in Section \ref{Granger_Causality}, after obtaining the estimated window causal matrix and the transform functions between variables, we can combine the outputs from the time-invariance and mechanism-invariance blocks and using $\hat{\mathcal{W}}$-selected column summation to predict $\hat{\mathcal{X}}$. We consider that the time-invariance block facilitates the identification of parent-child relationships between variables, formalized as $\hat{\mathcal{W}}$, while the mechanism-invariance block helps to explore the generative mechanisms, i.e., transform functions. Consequently, we can naturally combine the outputs $\hat{\mathcal{W}}$ and $\overline{W}$. Specifically, by utilizing $\hat{\mathcal{W}}$ and the computed $\overline{W}_\psi,\psi=1,...,c$, we can ultimately obtain the estimates $\hat{\mathcal{X}}^{\hat{\tau}+\psi}$, namely $\hat{\mathcal{W}}$-selected column summation,

\begin{equation}
  \hat{\mathcal{X}}^{\hat{\tau}+\psi}=\sum_{\tau=0}^{\overline{\tau}} \overline{W}_\psi^{\tau} \odot \hat{\mathcal{W}}^{\tau}
  \label{E5}
\end{equation}

Here, we need to consider each $\tau\in\{0,...,\overline{\tau}\}$ and combine the estimated window causal matrix $\hat{\mathcal{W}}$ with the corresponding transformed window observations $\overline{W}_\psi$ obtained through $B_m$ to obtain the values of $\hat{\mathcal{X}}^{\hat{\tau}+\psi}$. Our ultimate goal is to find a window causal matrix $\hat{\mathcal{W}}$ that satisfies the conditions by optimizing the squared error loss (MSE) $\mathcal{L}$ between the predicted $\hat{\mathcal{X}}$ and the ground truth $\mathcal{X}$ at each time point $t$. The final auto-regressive equation is expressed as follows:

\begin{equation}
  \mathcal{L}=\sum_{t=\hat{\tau}+1}^{T}\sum_{i=1}^{d}||X_i^{t}-\hat{\mathcal{X}}_i^{t}||^2_2
  \label{E6}
\end{equation}

We adopt the gradient $\bigtriangledown \mathcal{L}$ to optimize the parameters within the time-invariance and mechanism-invariance blocks.

\section{Experiment Results \label{Experiment_Results}}

In this section, we present a comprehensive series of experiments on both synthetic and benchmark datasets to verify the effectiveness of the proposed STIC. Following the experimental setup of \citep{runge2019detecting, runge2020discovering}, we compare STIC against the \textbf{constraint-based approaches} such as PCMCI \citep{runge2019detecting} and PCMCI+ \citep{runge2020discovering}, the \textbf{score-based approaches} such as DYNOTEARS \citep{pamfil2020dynotears}, and the \textbf{Granger-based approaches} TCDF \citep{nauta2019causal}, CUTS \citep{cheng2022cuts} and CUTS+ \citep{cheng2023cuts+}.

Our causal discovery algorithm is implemented using PyTorch. The source code for our algorithm is publicly available at the following URL \footnote{\url{https://github.com/HITshenrj/STIC}}. Both the time-invariance block and mechanism-invariance block are implemented using convolutional neural networks. 

Firstly, we conducted experiments on synthetic datasets, encompassing both linear and non-linear cases. The methods of generating synthetic datasets for both linear and non-linear cases will be introduced separately in Section \ref{Synthetic_Data_exp}. Secondly, we proceeded to perform experiments on benchmark datasets to demonstrate the practical value of our model in Section \ref{benchmark}. Thirdly, to evaluate the sensitivity of hyper-parameters, such as the learning rate (default $1e^{-5}$), the predefined $\overline{\tau}$ (default $0.4d$) and the threshold $p$ (default $0.3$), we conducted ablation experiments as detailed in Section \ref{ablation_study}.

We employ two kinds of evaluation metrics to assess the quality of the estimated causal matrix: the F1 score and precision. A higher F1 score indicates a more comprehensive estimation of the window causal matrix, while a higher precision indicates the ability to identify a larger number of causal edges. In this paper, we consider causal edges with different time lags for the same pair of variables as distinct causal edges. Specifically, if there exists a causal edge from $X_{i}$ to $X_{j}$ with a time lag of $\tau_1$, and another causal edge from $X_{i}$ to $X_{j}$ with the time lags of $\tau_2$, where $i\neq j$ and $\tau_1\neq \tau_2$, we regard these as two separate causal edges. Due to the need to predefine the maximum time lag in STIC, we truncate the estimated $\hat{\mathcal{W}}\in\mathbb{R}^{d\times d\times (\overline{\tau}+1)}$ to $\hat{\mathcal{W}}\in\mathbb{R}^{d\times d\times (\widetilde{\tau}+1)}$ and then compute the evaluation metrics. We handle other baselines (such as PCMCI, PCMCI+, DYNOTEARS, CUTS, CUTS+) requiring a predefined maximum time lag parameter in the same manner.

\subsection{Baselines}

We select six state-of-the-art causal discovery methods as baselines for comparison:

\begin{itemize}
  \item PCMCI \citep{runge2019detecting} is a notable work that extends the PC algorithm \citep{kalisch2007estimating} for causal discovery from time-series data. The source code for PCMCI is available at \url{https://github.com/jakobrunge/tigramite}. PCMCI divides the causal discovery process into two components: the identification of relevant sets through conditional independence tests and the direction determination. It assumes causal stationarity, the absence of contemporaneous causal links, and no hidden variables. Specifically, the PC-stable algorithm \citep{colombo2014order} is employed to remove irrelevant conditions through iterative independence tests. Furthermore the Multivariate Conditional Independence test is conducted to address false-positive control in scenarios with highly interdependent time series.
  \item PCMCI+ \citep{runge2020discovering} improves upon PCMCI by reducing the number of independence tests and optimizing the selection of conditional sets, resulting in superior effectiveness and efficiency in the same experimental setting. The source code is also available at \url{https://github.com/jakobrunge/tigramite}. PCMCI+ overcomes the limitation of the ``no contemporaneous causal links" assumption in PCMCI. PCMCI+  expedites the selection of conditional sets by testing all time-lagged pairs conditional on only the strongest $p$ adjacencies in each $p$-iteration, without evaluating all $p$-dimensional subsets of adjacencies. Moreover, intra-slice sets are introduced to further refine the determination of all structures.
  \item DYNOTEARS \citep{pamfil2020dynotears} represents a groundbreaking advancement in the field of causal discovery from time-series data by transforming the combinatorial graph search problem into a continuous optimization problem. The details of this work can be found in the repository located at \url{https://github.com/ckassaad/causal_discovery_for_time_series}. This approach characterizes the acyclicity constraint as a smooth equality constraint through the minimization of a penalized loss while adhering to the acyclicity constraint.
  \item TCDF \citep{nauta2019causal} is an outstanding work that utilizes attention-based convolutional neural networks (CNNs) to explore causal relationships between time series and the time delay between cause and effect. The code for TCDF can be accessed at \url{https://github.com/M-Nauta/TCDF}. By leveraging Granger causality, TCDF predicts one time series based on other time series and its own historical values, employing CNNs to identify and analyze causal relationships within time-series data.
  \item CUTS \citep{cheng2022cuts} is an outstanding neural Granger causal discovery algorithm for jointly imputing unobserved data points and building causal graphs, by incorporating two mutually boosting modules (latent data prediction and causal graph fitting) in an iterative framework. fter hallucinating and registering unstructured data, which might be of high dimension and with complex distribution, CUTS builds a causal adjacency matrix with imputed data under sparse penalty. The code for CUTS is available at \url{https://github.com/jarrycyx/UNN/tree/main/CUTS}. CUTS is a promising step toward applying causal discovery to real-world applications with non-ideal observations.
  \item CUTS+ \citep{cheng2023cuts+} is built on the Granger-causality-based causal discovery method CUTS and increases scalability through coarse-to-fine discovery and message-passing-based methods. The code for CUTS+ can be accessed at \url{https://github.com/jarrycyx/UNN/tree/main/CUTS_Plus}. CUTS+ significantly improves causal discovery performance on high-dimensional data with various types of irregular sampling.
\end{itemize}

\subsection{Experiments on Synthetic Datasets \label{Synthetic_Data_exp}}

\begin{figure}[!t]
  \centering
  \includegraphics[width=0.6\linewidth]{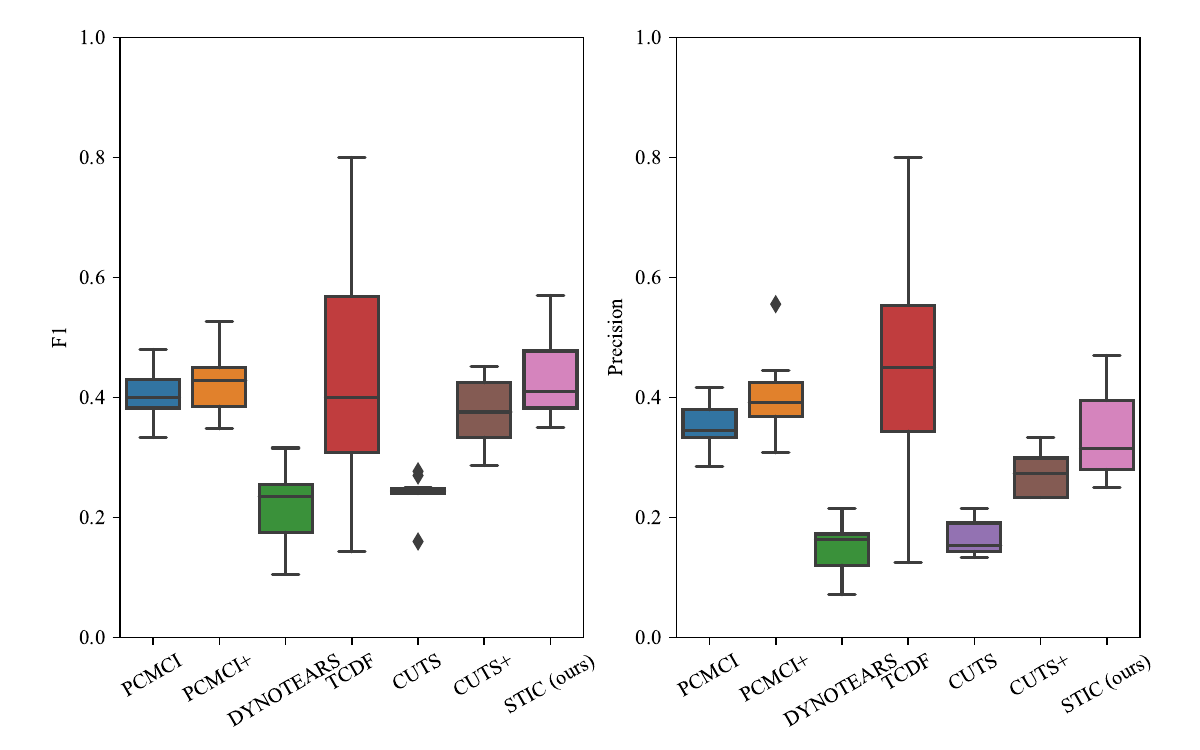}

  \caption{The results of F1 and precision evaluated on nonlinear Gaussian datasets. We fix numbers of variables $d=5$ and observed time steps $T=1000$.}
  \label{fig5}
\end{figure}

We generate synthetic datasets in the following manner. Firstly, we consider several typical challenges \citep{runge2019detecting,runge2020discovering} with contemporaneous and time-lagged causal dependencies, following an additive noise model. We set the ground truth maximum time lag to $0.4d$ and initialize the existence of each edge in the true window causal matrix $\mathcal{W}$ with a probability of 50\%. For each variable $X_i$, its relation to with its parents $Pa_{\mathcal{G}}(X_i)$ is defined as $X_i=f_i(Pa_{\mathcal{G}}(X_i))+\varepsilon_i$, where $f_i$ represents the ground truth transformation function between $X_i$'s parents $Pa_{\mathcal{G}}(X_i)$ and $X_i$. If $X_j\in Pa^{\tau}_{\mathcal{G}}(X_i)$, then in the ground truth causal matrix $\mathcal{W}$, $\mathcal{W}_{ji}^ {\tau}=1$. Secondly, for linear datasets, each $f_i$ is defined by a weighted linear function, while for nonlinear datasets, each $f_i$ is defined using a weighted cosine function. We sample the weights from a uniform distribution, such that if a causal edge exists, the corresponding weight in the additive noise model is sampled from the interval $U(-2,-0.5]\cup [0.5,2)$ to ensure non-zero values. For non-causal edges, the weight is set to 0. The noise term $\varepsilon_i$ follows either a standard normal distribution $\mathcal{N}(0,1)$ or is uniformly sampled from the interval $U[0,1]$. These data-generating procedures are similar to those used by the PCMCI family \citep{runge2019detecting,runge2020discovering} and CUTS family \citep{cheng2022cuts,cheng2023cuts+}.

In the following, we present different results on linear Gaussian datasets (Section \ref{LG}), nonlinear Gaussian datasets (Section \ref{NLG}), and linear uniform datasets (Section \ref{LU}) to demonstrate the superiority of our model. Specifically, to reduce the impact of random initialization, we conduct $10$ experiments for each type of datasets and report the mean and variance of the experimental results.

\subsubsection{\textbf{Linear Gaussian Datasets}\label{LG}}

The data generation process for linear Gaussian datasets follows the relationship $X_i=w_i Pa_{\mathcal{G}}(X_i)+\varepsilon_i$, where $\varepsilon_i$ is sampled from a standard normal distribution $\mathcal{N}(0,1)$. To demonstrate the capability of our model in causal discovery from time-series data on datasets of varying sizes, we compare STIC with baselines under different conditions, including different numbers of variables ($d=\{5,10,15,20\}$) and different lengths of time steps ($T=\{100, 200, 500, 1000\}$). 

The results are summarized in Figure \ref{fig3}. Figure \ref{fig3} left presents the variation of F1 score as the number of variables increases, while Figure \ref{fig3} right shows the variation of precision with the number of variables. A comprehensive analysis of the experiments requires the joint consideration of both Figure \ref{fig3} left and right. From a macroscopic perspective, our proposed STIC achieves the highest F1 scores on linear Gaussian datasets, while precision reaches the state-of-the-art levels in most cases. We will compare the performance of STIC and the baselines from two aspects of analysis:\\
\textbf{Aspect 1: The relationship between the number of variables $d$ and the model when $T$ remains constant.} When the observed time steps is fixed at $T=1000$, corresponding to the top-left graphs in Figure \ref{fig3} left and right, we observe that as the number of variables increases, the F1 scores of all causal discovery methods tend to decrease. However, our proposed STIC achieves an average F1 score of 0.86, 0.77, 0.79, 0.77 and an average precision of 0.80, 0.62, 0.74, 0.72 across the four different numbers of variables, surpassing other strong baselines. By comparing the line plots in the corresponding positions of Figure \ref{fig3} left and right, especially when $T=100$, corresponding to the bottom-right graphs in Figure \ref{fig3} left and right, we find that our proposed STIC achieves an average F1 score of 0.76, 0.76, 0.77, 0.65 and an average precision of 0.66, 0.70, 0.77, 0.65 across the four different numbers of variables, significantly outperforming other strong baselines. 

In the case of fixed observed time steps, as the number of variables increases, constraint-based approaches such as PCMCI and PCMCI+ suffer from severe performance degradation because they require significant prior knowledge involvement in determining the threshold $p$, which determines the presence of causal edges. For score-based methods, the DYNOTEARS method shows relatively stable performance as the number of variables increases, but it does not achieve the optimal performance among all methods. As for Granger-based methods, CUTS and CUTS+ often suffer from poor performance due to the inability to recognize time lags. Our proposed STIC and the TCDF method achieve competitive results in terms of F1 scores. However, our method exhibits higher precision. 


\begin{figure}[!t]
  \centering
  \subfigure{\includegraphics[width=0.6\linewidth]{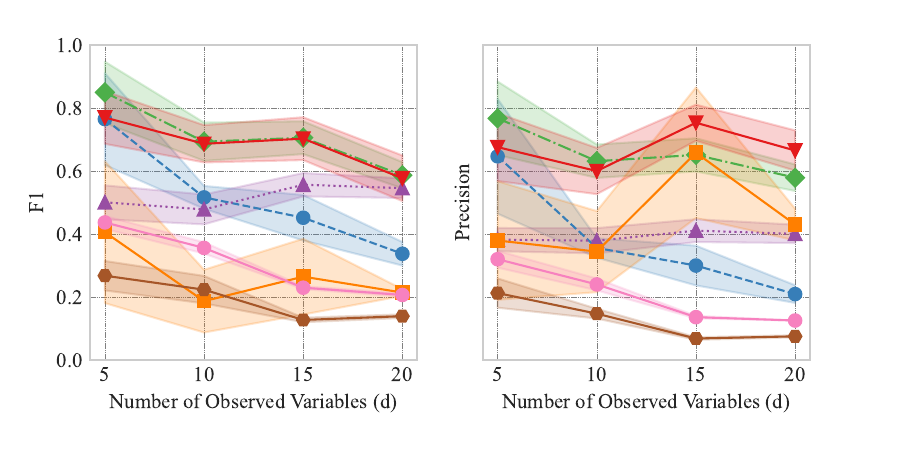}%
  \label{fig4:a}}\vspace{-5mm}

  \subfigure{\includegraphics[width=0.6\linewidth]{legend.pdf}%
  \label{fig4:b}}
  \caption{The results of F1 and precision evaluated on linear uniform datasets with fixed observed time steps ($T=1000$) and the number of variables ($d$) ranging from 5 to 20.}
  \label{fig4}
\end{figure}

We attribute this superior performance to the window representation employed in STIC. By repeatedly extracting features from observed time series in different window observations, such representation acts as a form of data augmentation and aggregation. It enables a macroscopic view of common characteristics among multiple window observations, facilitating the learning of more accurate causal structures. Thus, our STIC model achieves optimal performance when the number of variables $d$ changes.\\
\textbf{Aspect 2: The relationship between the observed time steps $T$ and the model when $d$ remains constant.} When examining the impact of observed time steps $T$ on the models while keeping the number of variables constant, we observe that our STIC method consistently maintains an F1 score of approximately 0.7 across different values of $T$. However, PCMCI+ and DYNOTEARS exhibit a significant decline in their F1 scores as $T$ decreases. For instance, at $T=1000$, PCMCI+ and DYNOTEARS perform similarly to our STIC method, but at $T=100$, their F1 scores drop to half of that achieved by our STIC method. For PCMCI, it consistently falls behind our STIC method, regardless of changes in $T$. While TCDF achieves a relatively consistent level of performance, it exhibits lower performance compared to our model. Furthermore, we find that after treating different time lags as different causal edges, the F1 scores and precisions of CUTS and CUTS+ are maintained at a relatively low level.

For constraint-based approaches, the PCMCI and PCMCI+ algorithms perform poorly because as the number of samples decreases, the statistical significance of conditional independence cannot fully capture the causal relationships between variables. As for score-based methods, DYNOTEARS does not perform well on linear data. One possible reason is that DYNOTEARS heavily relies on acyclicity in its search, which may not converge to the correct causal graph. Regarding Granger-based methods, we believe that they are overly conservative and fail to accurately predict all correct causal edges. 

In contrast, our STIC model is capable of predicting a greater number of causal edges, which is crucial for discovering new knowledge. This superior performance can be attributed to the design of the convolutional time-invariance block. This design allows for the extraction of more causal structure features from limited observed data, enabling a more accurate exploration of potential causal relationships even with a small number of samples. Consequently, our STIC model effectively addresses the challenge of causal discovery in low-sample scenarios, i.e., improving sample efficiency.


\begin{table*}[!t]
  \centering
  \caption{The results of F1 and precision evaluated on FMRI dataset. In terms of average of both F1 and precision, STIC outperforms the other baselines. Moreover, STIC shows a better stablity in view of variance.}
  \label{tab_Simulated}
  \resizebox{\textwidth}{!}{%
  \begin{tabular}{ccccccccc}
  \toprule[1pt]
   & \multicolumn{1}{c|}{} & PCMCI & PCMCI+ & DTNOTEARS & TCDF & CUTS & CUTS+ & STIC (ours) \\ \midrule
  \multicolumn{1}{c|}{\multirow{2}{*}{$d=5$}} & \multicolumn{1}{c|}{F1} & 0.38$\pm$0.004 & 0.37$\pm$0.009 & 0.43$\pm$0.009 & 0.42$\pm$0.007 & 0.35$\pm$0.001 & 0.36$\pm$0.034 & \textbf{0.45$\pm$0.003}$\uparrow$ \\
  \multicolumn{1}{c|}{} & \multicolumn{1}{c|}{precision} & 0.31$\pm$0.014 & 0.33$\pm$0.016 & 0.31$\pm$0.017 & 0.44$\pm$0.006 & 0.22$\pm$0.007 & 0.29$\pm$0.035 & \textbf{0.70$\pm$0.030}$\uparrow$ \\ \midrule
  \multicolumn{1}{c|}{\multirow{2}{*}{$d=10$}} & \multicolumn{1}{c|}{F1} & 0.24$\pm$0.001 & 0.31$\pm$0.001 & 0.20$\pm$0.002 & 0.42$\pm$0.001 & 0.11$\pm$0.001 & 0.44$\pm$0.001 & \textbf{0.47$\pm$0.001}$\uparrow$ \\
  \multicolumn{1}{c|}{} & \multicolumn{1}{c|}{precision} & 0.34$\pm$0.001 & 0.37$\pm$0.00 & 0.33$\pm$0.004 & 0.44$\pm$0.003 & 0.20$\pm$0.000 & 0.51$\pm$0.001 & \textbf{0.60$\pm$0.023}$\uparrow$  \\ \midrule
  \multicolumn{1}{c|}{\multirow{2}{*}{$d=15$}} & \multicolumn{1}{c|}{F1} & 0.27$\pm$0.003 & 0.35$\pm$0.007 & 0.19$\pm$0.011 & 0.35$\pm$0.002 & 0.14$\pm$0.001 & 0.26$\pm$0.020 & \textbf{0.53$\pm$0.006}$\uparrow$ \\
  \multicolumn{1}{c|}{} & \multicolumn{1}{c|}{precision} & 0.19$\pm$0.002 & 0.20$\pm$0.001 & 0.11$\pm$0.010 & 0.43$\pm$0.032 & 0.08$\pm$0.002 & 0.26$\pm$0.003 & \textbf{0.80$\pm$0.005}$\uparrow$ \\
  \bottomrule[1pt]
  \end{tabular}%
  }
\end{table*}

\subsubsection{\textbf{Nonlinear Gaussian Datasets}\label{NLG}}

In this section, we perform experiments on nonlinear Gaussian datasets to evaluate the performance of STIC. We set the number of variables $(d=5)$ and the observed time steps $(T=1000)$. For each $X_i$, its relationship with its parents $Pa_{\mathcal{G}}(X_i)$ is defined using the cosine function, and the noise term $\varepsilon_i$ follows the standard normal distribution.

The performance of STIC and the baselines is visualized in Figure \ref{fig5}. It can be observed that STIC achieves an F1 score of 0.44, which is higher than all baselines (PCMCI: 0.41, PCMCI+: 0.43, DYNOTEARS: 0.22, TCDF: 0.41, CUTS: 0.24, CUTS+: 0.37). It can be seen that STIC achieves a higher F1 score despite having lower precision compared to the other baselines. For constraint-based methods (PCMCI and PCMCI+), one possible reason for achieving similar F1 scores with our proposed STIC is that the length of observed time steps is set to 1000, which is sufficient for statistical independence tests. Thus, the conditional independence tests can directly operate on the data without being affected by noise. Regarding score-based methods, we believe that DYNOTEARS uses a simple network that may not effectively capture nonlinear transforms, leading to lower F1 scores. For Granger-based methods, although TCDF achieves a comparable F1 score to STIC (and even higher precision), the variance of STIC is significantly lower. This indicates that TCDF is highly unstable, and there is a considerable amount of uncertainty in causal discovery. One possible reason for this is that TCDF does not incorporate window representation like STIC, which could lead to inefficient training of the convolutional neural network. We find that CUTS and CUTS+ are not very good at causal discovery on nonlinear Gaussian datasets, and both $F1$ and precision are lower than our STIC. One possible reason is that both models rely on graph neural networks and treat learnable graph structures as estimated causal graphs. However, the graph structure in the graph neural network is full of correlational relationships rather than causal relationships, so the output graph structure does not contain only causal relationships, resulting in a decrease in both F1 and precision. We believe that the robustness of our proposed STIC lies in the mechanism-invariance block, which repeatedly verifies the functional causal relationships within each single window, effectively reducing model instability.

\subsubsection{\textbf{Linear Uniform Datasets}\label{LU}}

The linear uniform datasets is generated with observed time steps ($T=1000$) by varied numbers of variables ($d=\{5,10,15,20\}$). For each $X_i$, $f_i$ is set as a linear function, while the noise term $\varepsilon_i$ follows a uniform distribution $U[0,1]$.

The performance of STIC and baselines are shown in Figure \ref{fig4}. STIC outperforms baselines in terms of F1 score and precision in most cases, especially when the number of time series $d$ is large. For constraint-based methods, PCMCI and PCMCI+ perform poorly in terms of F1 score and precision when the number of variables is relatively large ($d=\{10,15,20\}$). We consider that since conditional independence tests serve as strict indicators of causal relationships, they may fail due to the limited number of time steps and the presence of uniform noise. Moreover, PCMCI cannot determine intra-slice causal relationships and performs much worse than our STIC model in terms of F1 score and precision. For score-based methods, DYNOTEARS identifies causal relationships by fitting auto-regressive coefficients between variables, treating them as estimated causal relationships. However, due to the strong influence of noise, DYNOTEARS fails to recognize causal relationships in the linear uniform datasets. Interestingly, TCDF, which shows competitive performance compared to our STIC model in Section \ref{LG}, performs particularly poorly on the linear uniform datasets. From the high precision and low F1 score of TCDF, we can deduce that the uniform distribution introduces many incorrectly estimated causal edges during the process of estimating temporal causality based on Granger causality using past value of other variables. The F1 score and precision of CUTS and CUTS+ further support the idea that Granger causality is not well applicable to linear uniform datasets. One possible reason is that for a uniform distribution, its inverse transformation equation still exists, which leads to the performance degradation of finding causality from correlation.

\subsection{Experiments on Benchmark Datasets\label{benchmark}}

In this section, we utilize FMRI benchmark datasets, a common neuroscientific benchmark dataset called Functional Magnetic Resonance Imaging \citep{smith2011network}, to explore and discover brain blood flow patterns. The dataset contains 28 different underlying brain networks with the number of observed variables ($d=\{5,10,15\}$). For each of the 28 brain networks, we observe 200 time steps for causal discovery. The results are reported in Table \ref{tab_Simulated}.

The results demonstrate that STIC achieves the highest average F1 scores on all kinds of observed variables, surpassing the average F1 scores of PCMCI, PCMCI+, DYNOTEARS, TCDF, CUTS and CUTS+. Moreover, in terms of precision, STIC achieves significantly higher precisions than those of the other baselines. For constraint-based methods, such as PCMCI and PCMCI+, their poor performance on the FMRI datasets may be attributed to the short length of observed time steps, which affects their ability to accurately test for conditional independence. Regarding DYNOTEARS, we believe that acyclicity regularizers still limit its performance. In comparison, our STIC model outperforms TCDF, CUTS and CUTS+ by utilizing a window representation, which enhances the representation of observed data within each window. This enables more accurate learning of common causal features and structures across multiple windows.

\subsection{Ablation Study\label{ablation_study}}




\begin{table}[]
    \centering
    \caption{The results of ablation study on the linear Gaussian datasets with the number of variables ($d=5$).}
    \resizebox{\columnwidth}{!}{%
    \begin{tabular}{c|ccc|ccc|ccc}
    \hline
     & \multicolumn{3}{c|}{Learning Rate} & \multicolumn{3}{c|}{Max Time Lag} & \multicolumn{3}{c}{Threshold} \\
     & lr=$1e^{-4}$ & lr=$1e^{-5}$ & lr=$1e^{-6}$ & $\overline{\tau}=2$ & $\overline{\tau}=3$ & $\overline{\tau}=4$ & $p=0.1$ & $p=0.3$ & $p=0.5$ \\ \hline
    F1 & 0.77$\pm$0.005 & 0.76$\pm$0.008 & 0.78$\pm$0.004 & 0.76$\pm$0.008 & 0.68$\pm$0.004 & 0.60$\pm$0.003 & 0.43$\pm$0.001 & 0.76$\pm$0.008 & 0.80$\pm$0.020 \\
    precision & 0.66$\pm$0.006 & 0.66$\pm$0.013 & 0.68$\pm$0.016 & 0.66$\pm$0.013 & 0.53$\pm$0.005 & 0.45$\pm$0.003 & 0.27$\pm$0.001 & 0.66$\pm$0.013 & 0.89$\pm$0.019 \\ \hline
    \end{tabular}%
    }
    \label{tab_ablation}
\end{table}

We conduct ablation experiments on the linear Gaussian datasets with the number of variables ($d=5$), to investigate the impact of different hyper-parameters on the experimental results, such as the learning rate (default: $1e^{-5}$), the predefined maximum time lag (default: $0.4d=2$), and the threshold $p$ (default: 0.3). Specifically, we vary the learning rate by increasing it to $1e^{-4}$ and decreasing it to $1e^{-6}$. We also increased the predefined maximum lag to $\overline{\tau}=3$ and $\overline{\tau}=4$, respectively, and change the threshold to $p=0.1$ or $p=0.5$. The empirical results are summarized in Table \ref{tab_ablation}.

\begin{itemize}
  \item \textbf{The learning rate lr}: The experiments reveal that manipulating the learning rate, either by increasing or decreasing it, has little effect on the F1 score and precision. This finding suggests that our convolutional neural network architecture is not sensitive to changes in the learning rate, simplifying the parameter tuning process.
  \item  \textbf{The predefined maximum time lag $\overline{\tau}$}: However, increasing the predefined maximum lag $\overline{\tau}$ gradually deteriorates performance. We speculate that this decline occurs because, with a longer lag, the window for observations expands, potentially causing STIC to learn multi-hop causal edges ($X_{i}\stackrel{\tau_1}{\longrightarrow} X_{j}\stackrel{\tau_2}{\longrightarrow} X_{k}$) as single-hop causal edges ($X_i\stackrel{\tau_1+\tau_2}{\longrightarrow} X_{k}$). Addressing this issue could be a focus for future research. 
  \item  \textbf{The threshold $p$}: Furthermore, comparing the default setting to STIC with $p=0.1$, we observe a significant decline in both F1 score and precision when the threshold is lower. When comparing the default setting to STIC with $p=0.5$, we find that while the F1 score remains relatively stable, precision notably improves when the threshold is increased. These findings indicate that reducing the threshold adversely affects the model's ability to explore causal edges, while setting a higher threshold may cause the model to consider nearly all estimated edges as causal, resulting in increased precision but a similar F1 score. Thus, the threshold plays a pivotal role in discovering more causal edges, and a trade-off needs to be made.
\end{itemize}

\section{Discussion}

This study presents two kinds of short-term invariance-based convolutional neural networks for discovery causality  from time-series data. Major findings include: (1) our methods, based on gradients, effectively discover causality from time-series data; 
(2) convolutional neural network based on short-term invariance improves the sample efficiency of causal discovery. (3) our proposed STIC demonstrates significantly superior performance compared to baseline causal discovery algorithms. In this section, we discuss these results in detail.

\subsection{What contributes to the effectiveness of STIC?}

\subsubsection{Why can STIC find causal relationships}
Numerous gradient-based methods have been developed, such as DYNOTEARS within score-based approaches \citep{pamfil2020dynotears}, and TCDF \citep{nauta2019causal}, CUTS\citep{cheng2022cuts} and CUTS+\citep{cheng2023cuts+} within Granger-based approaches. Including our proposed STIC, these gradient-based methods aim to optimize estimated causal matrix by maximizing or minimizing constrained functions. With the rapid advancement and widespread adoption of deep Neural Networks (NNs), researchers have begun employing NNs to infer nonlinear Granger causality, demonstrating the effectiveness of gradient-based methods in causal discovery \citep{tank2021neural,wu2021granger,khanna2019economy}. In our approach, we maintain the assumption and the constrained functions of Granger causality so that our method remains effective in discovering causal relationships.

\subsubsection{Why can STIC find the true causality} As time progresses, the values of observed variables change due to statistical shifts in distributions. However, the causal relationships between the variables remain the same. For example, carbohydrate intake may lead to an increase in blood glucose, but the specific magnitude of the increase may vary with covariates such as body weight. The ``lead'' property is used as an indicator of causal relationships, i.e., invariance \citep{magliacane2018domain,rojas2018invariant,santos2021domain,li2021revisiting}. In this paper, we observe that some causal relationships may also vary over time. Therefore, we make a more reasonable assumption, namely short-term time invariance and mechanism invariance \citep{entner2010causal,liu2023causal,zhang2017causal}. Building on these two forms of short-term invariance, we posit that both the window causal matrix $\mathcal{W}$ and the transform functions $f$ remain unchanged in the short term. For example, within a few days (short-term), since covariates affecting blood glucose levels, such as body weight, remain nearly constant, the increase in blood glucose levels due to carbohydrate intake is also essentially constant. The short-term mechanism invariance proposed in this paper is also considered an invariant principle \citep{liu2022towards}. Building on these forms of invariance, a natural extension is the introduction of parallel time-invariance and mechanism-invariance blocks for joint training, as proposed in this paper. Through the theoretical validation of convolution in Section \ref{Necessity_of_convolution}, we further affirm the applicability of convolution to causal discovery from time-series data. Additionally, the Granger causality is commonly employed to examine short-term causal relationships \citep{ahmad2005indian}, which further aligns with our assumptions. Under the premise of theoretical soundness and practical applicability, our STIC framework proves highly effective.

\subsection{What contributes to the exceptional
performance of STIC?}

\subsubsection{High F1 scores and precisions} The experiments conducted on both synthetic and FMRI benchmark datasets in Section \ref{Experiment_Results} demonstrate that our STIC model achieves the state-of-the-art F1 scores and precisions in most cases. We attribute the performance improvement to the incorporation of the window representation, the time-invariance block, and the mechanism-invariance block. The window representation serves as a form of data augmentation and aggregation, providing a macroscopic understanding of common features across multiple window observations, thereby facilitating the learning of more accurate causal structures. The time-invariance block extracts common features from multiple window observations and achieves effective information aggregation, enhancing sample efficiency and enabling the model to achieve high performance. The mechanism-invariance block, with nested convolution kernels, iteratively examines the functional transform within each individual window, enabling complex nonlinear transformations. With improved accuracy in both causal structures and complex nonlinear transformations, STIC demonstrates exceptional performance.

\subsubsection{High sample efficiency} The window representation, introduced in Section \ref{Window_Representation}, facilitates the segmentation of the entire observed dataset $\mathcal{X}\in\mathbb{R}^{d\times T}$ into $c=T-\overline{\tau}-1$ partitions, leveraging only a predefined hyperparameter $\overline{\tau}$. Each window observation $W_\psi$, where $\psi=1,...,c$, is ensured to satisfy both short-term time invariance and mechanism invariance simultaneously. This representation method, similar to batch training techniques \citep{liang2006fast,li2014efficient,hong2020graph} optimizes data utilization, thus enhancing sample efficiency. Moreover, another pivotal aspect contributing to sample efficiency is the novel invariance-based convolutional neural network design. This architecture enables the extraction of richer causal structure features from limited observed data, facilitating more accurate exploration of potential causal relationships even with a limited length of observed time steps. Consequently, our STIC model effectively tackles the challenge of causal discovery in low-sample scenarios, thereby improving sample efficiency.

\section{Conclusion, limitations and Future Works}

This paper introduces STIC, a novel method designed for causal discovery from time-series data by leveraging both short-term time invariance and mechanism invariance. STIC employs sliding windows in conjunction with convolutional neural networks to incorporate these two kinds of invariance, and then transforms the searching for window causal matrix into a continuous auto-regressive problem. The compatibility between causal structures in time series and convolutional neural networks is supported by our theoretical analysis, reinforcing the rationale behind STIC's design. Our experimental results on synthetic and benchmark datasets demonstrate the efficiency and stability of STIC, particularly when dealing with datasets that have shorter lengths of observed time steps. It showcases the effectiveness of the short-term invariance-based approach in capturing temporal causal structures.

However, STIC has certain limitations that require further investigation. Firstly, while STIC demonstrates effectiveness under this assumption, it becomes constrained when faced with non-additive noise. Future research should aim to develop more comprehensive approaches capable of handling various types of non-additive noise. Secondly, the manual predefined maximum lag used in STIC may pose limitations. Ablation experiments indicate that this hyperparameter setting can lead to the model learning multi-hop causal edges as single-hop causal edges. To overcome this, future research should explore more advanced blocks, such as attention mechanisms, to further enhance the performance of STIC.

In summary, STIC represents a promising research direction for addressing the challenges of causal discovery from time-series data, and we hope that STIC can discover new causal knowledge and provide new research ideas for medical and other fields.

\section*{Acknowledgments}

This study was supported in part by a grant from the National Key Research and Development Program of China [2021ZD0110900] and the National Natural Science Foundation of China [72293584].

\bibliographystyle{unsrtnat}
\bibliography{references}

\begin{thebibliography}{55}
\providecommand{\natexlab}[1]{#1}
\providecommand{\url}[1]{\texttt{#1}}
\expandafter\ifx\csname urlstyle\endcsname\relax
  \providecommand{\doi}[1]{doi: #1}\else
  \providecommand{\doi}{doi: \begingroup \urlstyle{rm}\Url}\fi

\bibitem[Cowls and Schroeder(2015)]{cowls2015causation}
Josh Cowls and Ralph Schroeder.
\newblock Causation, correlation, and big data in social science research.
\newblock \emph{Policy \& Internet}, 7\penalty0 (4):\penalty0 447--472, 2015.

\bibitem[Pawlowski et~al.(2020)Pawlowski, Coelho~de Castro, and Glocker]{pawlowski2020deep}
Nick Pawlowski, Daniel Coelho~de Castro, and Ben Glocker.
\newblock Deep structural causal models for tractable counterfactual inference.
\newblock \emph{Advances in Neural Information Processing Systems}, 33:\penalty0 857--869, 2020.

\bibitem[Chan et~al.(2024)Chan, Yiu, Kim, and Abu-Siada]{chan2024fuzzy}
Kit~Yan Chan, Ka~Fai~Cedric Yiu, Dowon Kim, and Ahmed Abu-Siada.
\newblock Fuzzy clustering-based deep learning for short-term load forecasting in power grid systems using time-varying and time-invariant features.
\newblock \emph{Sensors}, 24\penalty0 (5):\penalty0 1391, 2024.

\bibitem[Entner and Hoyer(2010)]{entner2010causal}
Doris Entner and Patrik~O Hoyer.
\newblock On causal discovery from time series data using fci.
\newblock \emph{Probabilistic graphical models}, pages 121--128, 2010.

\bibitem[Pearl(2009)]{pearl2009causality}
Judea Pearl.
\newblock \emph{Causality}.
\newblock Cambridge university press, 2009.

\bibitem[Runge et~al.(2019)Runge, Nowack, Kretschmer, Flaxman, and Sejdinovic]{runge2019detecting}
Jakob Runge, Peer Nowack, Marlene Kretschmer, Seth Flaxman, and Dino Sejdinovic.
\newblock Detecting and quantifying causal associations in large nonlinear time series datasets.
\newblock \emph{Science advances}, 5\penalty0 (11):\penalty0 eaau4996, 2019.

\bibitem[Runge(2020)]{runge2020discovering}
Jakob Runge.
\newblock Discovering contemporaneous and lagged causal relations in autocorrelated nonlinear time series datasets.
\newblock In \emph{Conference on Uncertainty in Artificial Intelligence}, pages 1388--1397. PMLR, 2020.

\bibitem[Pamfil et~al.(2020)Pamfil, Sriwattanaworachai, Desai, Pilgerstorfer, Georgatzis, Beaumont, and Aragam]{pamfil2020dynotears}
Roxana Pamfil, Nisara Sriwattanaworachai, Shaan Desai, Philip Pilgerstorfer, Konstantinos Georgatzis, Paul Beaumont, and Bryon Aragam.
\newblock Dynotears: Structure learning from time-series data.
\newblock In \emph{International Conference on Artificial Intelligence and Statistics}, pages 1595--1605. PMLR, 2020.

\bibitem[Nauta et~al.(2019)Nauta, Bucur, and Seifert]{nauta2019causal}
Meike Nauta, Doina Bucur, and Christin Seifert.
\newblock Causal discovery with attention-based convolutional neural networks.
\newblock \emph{Machine Learning and Knowledge Extraction}, 1\penalty0 (1):\penalty0 312--340, 2019.

\bibitem[Cheng et~al.(2022)Cheng, Yang, Xiao, Li, Suo, He, and Dai]{cheng2022cuts}
Yuxiao Cheng, Runzhao Yang, Tingxiong Xiao, Zongren Li, Jinli Suo, Kunlun He, and Qionghai Dai.
\newblock Cuts: Neural causal discovery from irregular time-series data.
\newblock In \emph{The Eleventh International Conference on Learning Representations}, 2022.

\bibitem[Cheng et~al.(2023)Cheng, Li, Xiao, Li, Suo, He, and Dai]{cheng2023cuts+}
Yuxiao Cheng, Lianglong Li, Tingxiong Xiao, Zongren Li, Jinli Suo, Kunlun He, and Qionghai Dai.
\newblock Cuts+: High-dimensional causal discovery from irregular time-series.
\newblock \emph{arXiv preprint arXiv:2305.05890}, 2023.

\bibitem[Zhang et~al.(2011)Zhang, Peters, Janzing, and Sch{\"o}lkopf]{zhang2011kernel}
K~Zhang, J~Peters, D~Janzing, and B~Sch{\"o}lkopf.
\newblock Kernel-based conditional independence test and application in causal discovery.
\newblock In \emph{27th Conference on Uncertainty in Artificial Intelligence (UAI 2011)}, pages 804--813. AUAI Press, 2011.

\bibitem[Zhang and Suzuki(2023)]{zhang2023extending}
Bingyuan Zhang and Joe Suzuki.
\newblock Extending hilbert--schmidt independence criterion for testing conditional independence.
\newblock \emph{Entropy}, 25\penalty0 (3):\penalty0 425, 2023.

\bibitem[Spirtes and Zhang(2016)]{spirtes2016causal}
Peter Spirtes and Kun Zhang.
\newblock Causal discovery and inference: concepts and recent methodological advances.
\newblock In \emph{Applied informatics}, volume~3, pages 1--28. SpringerOpen, 2016.

\bibitem[Wang and Michoel(2017)]{wang2017efficient}
Lingfei Wang and Tom Michoel.
\newblock Efficient and accurate causal inference with hidden confounders from genome-transcriptome variation data.
\newblock \emph{PLoS computational biology}, 13\penalty0 (8):\penalty0 e1005703, 2017.

\bibitem[Varando(2020)]{varando2020learning}
Gherardo Varando.
\newblock Learning dags without imposing acyclicity.
\newblock \emph{arXiv preprint arXiv:2006.03005}, 2020.

\bibitem[Lippe et~al.(2021)Lippe, Cohen, and Gavves]{lippe2021efficient}
Phillip Lippe, Taco Cohen, and Efstratios Gavves.
\newblock Efficient neural causal discovery without acyclicity constraints.
\newblock In \emph{International Conference on Learning Representations}, 2021.

\bibitem[Zhang et~al.(2023)Zhang, Liu, Ma, Cai, Wang, and Chua]{zhang2023boosting}
An~Zhang, Fangfu Liu, Wenchang Ma, Zhibo Cai, Xiang Wang, and Tat-seng Chua.
\newblock Boosting differentiable causal discovery via adaptive sample reweighting.
\newblock \emph{arXiv preprint arXiv:2303.03187}, 2023.

\bibitem[Granger(1969)]{granger1969investigating}
Clive~WJ Granger.
\newblock Investigating causal relations by econometric models and cross-spectral methods.
\newblock \emph{Econometrica: journal of the Econometric Society}, pages 424--438, 1969.

\bibitem[Granger and Hatanaka(2015)]{granger2015spectral}
Clive William~John Granger and Michio Hatanaka.
\newblock \emph{Spectral Analysis of Economic Time Series.(PSME-1)}, volume 2066.
\newblock Princeton university press, 2015.

\bibitem[Cheng et~al.(2020)Cheng, Yang, Wang, Hu, Zhuang, and Song]{cheng2020time2graph}
Ziqiang Cheng, Yang Yang, Wei Wang, Wenjie Hu, Yueting Zhuang, and Guojie Song.
\newblock Time2graph: Revisiting time series modeling with dynamic shapelets.
\newblock In \emph{Proceedings of the AAAI conference on artificial intelligence}, volume~34, pages 3617--3624, 2020.

\bibitem[Cheng et~al.(2021)Cheng, Yang, Jiang, Hu, Ying, Chai, and Wang]{cheng2021time2graph+}
Ziqiang Cheng, Yang Yang, Shuo Jiang, Wenjie Hu, Zhangchi Ying, Ziwei Chai, and Chunping Wang.
\newblock Time2graph+: Bridging time series and graph representation learning via multiple attentions.
\newblock \emph{IEEE Transactions on Knowledge and Data Engineering}, 2021.

\bibitem[Spirtes et~al.(2000)Spirtes, Glymour, and Scheines]{spirtes2000causation}
Peter Spirtes, Clark~N Glymour, and Richard Scheines.
\newblock \emph{Causation, prediction, and search}.
\newblock MIT press, 2000.

\bibitem[Zhang and Spirtes(2002)]{zhang2002strong}
Jiji Zhang and Peter Spirtes.
\newblock Strong faithfulness and uniform consistency in causal inference.
\newblock In \emph{Proceedings of the Nineteenth conference on Uncertainty in Artificial Intelligence}, pages 632--639, 2002.

\bibitem[Robins et~al.(2003)Robins, Scheines, Spirtes, and Wasserman]{robins2003uniform}
James~M Robins, Richard Scheines, Peter Spirtes, and Larry Wasserman.
\newblock Uniform consistency in causal inference.
\newblock \emph{Biometrika}, 90\penalty0 (3):\penalty0 491--515, 2003.

\bibitem[Kalisch and B{\"u}hlman(2007)]{kalisch2007estimating}
Markus Kalisch and Peter B{\"u}hlman.
\newblock Estimating high-dimensional directed acyclic graphs with the pc-algorithm.
\newblock \emph{Journal of Machine Learning Research}, 8\penalty0 (3), 2007.

\bibitem[Assaad et~al.(2022)Assaad, Devijver, and Gaussier]{assaad2022survey}
Charles~K Assaad, Emilie Devijver, and Eric Gaussier.
\newblock Survey and evaluation of causal discovery methods for time series.
\newblock \emph{Journal of Artificial Intelligence Research}, 73:\penalty0 767--819, 2022.

\bibitem[Liu et~al.(2023)Liu, Sun, Hu, and Wang]{liu2023causal}
Mingzhou Liu, Xinwei Sun, Lingjing Hu, and Yizhou Wang.
\newblock Causal discovery from subsampled time series with proxy variables.
\newblock \emph{arXiv preprint arXiv:2305.05276}, 2023.

\bibitem[Zhang et~al.(2017)Zhang, Huang, Zhang, Glymour, and Sch{\"o}lkopf]{zhang2017causal}
Kun Zhang, Biwei Huang, Jiji Zhang, Clark Glymour, and Bernhard Sch{\"o}lkopf.
\newblock Causal discovery from nonstationary/heterogeneous data: Skeleton estimation and orientation determination.
\newblock In \emph{IJCAI: Proceedings of the Conference}, volume 2017, page 1347. NIH Public Access, 2017.

\bibitem[Mills and Granger(2013)]{mills2013granger}
Terence~C Mills and Clive Granger.
\newblock Granger: Spectral analysis, causality, forecasting, model interpretation and non-linearity.
\newblock \emph{A Very British Affair: Six Britons and the Development of Time Series Analysis During the 20th Century}, pages 288--342, 2013.

\bibitem[Shuman et~al.(2013)Shuman, Narang, Frossard, Ortega, and Vandergheynst]{shuman2013emerging}
David~I Shuman, Sunil~K Narang, Pascal Frossard, Antonio Ortega, and Pierre Vandergheynst.
\newblock The emerging field of signal processing on graphs: Extending high-dimensional data analysis to networks and other irregular domains.
\newblock \emph{IEEE signal processing magazine}, 30\penalty0 (3):\penalty0 83--98, 2013.

\bibitem[Sandryhaila and Moura(2013)]{sandryhaila2013discrete}
Aliaksei Sandryhaila and Jos{\'e}~MF Moura.
\newblock Discrete signal processing on graphs: Graph fourier transform.
\newblock In \emph{2013 IEEE International Conference on Acoustics, Speech and Signal Processing}, pages 6167--6170. IEEE, 2013.

\bibitem[Sardellitti et~al.(2017)Sardellitti, Barbarossa, and Di~Lorenzo]{sardellitti2017graph}
Stefania Sardellitti, Sergio Barbarossa, and Paolo Di~Lorenzo.
\newblock On the graph fourier transform for directed graphs.
\newblock \emph{IEEE Journal of Selected Topics in Signal Processing}, 11\penalty0 (6):\penalty0 796--811, 2017.

\bibitem[Zabusky(1968)]{zabusky1968solitons}
Norman~J Zabusky.
\newblock Solitons and bound states of the time-independent schr{\"o}dinger equation.
\newblock \emph{Physical review}, 168\penalty0 (1):\penalty0 124, 1968.

\bibitem[Rana and Liao(2019)]{rana2019time}
Jyotirmoy Rana and Shijun Liao.
\newblock On time independent schr{\"o}dinger equations in quantum mechanics by the homotopy analysis method.
\newblock \emph{Theoretical and Applied Mechanics Letters}, 9\penalty0 (6):\penalty0 376--381, 2019.

\bibitem[Zayed(1998)]{zayed1998convolution}
Ahmed~I Zayed.
\newblock A convolution and product theorem for the fractional fourier transform.
\newblock \emph{IEEE Signal processing letters}, 5\penalty0 (4):\penalty0 101--103, 1998.

\bibitem[Pavasant et~al.(2021)Pavasant, Numao, and Fukui]{pavasant2021spatio}
Nat Pavasant, Masayuki Numao, and Ken-ichi Fukui.
\newblock Spatio-temporal change detection using granger sequence pattern.
\newblock In \emph{Proceedings of the Twenty-Ninth International Conference on International Joint Conferences on Artificial Intelligence}, pages 5202--5203, 2021.

\bibitem[Nauta(2018)]{nauta2018temporal}
Meike Nauta.
\newblock Temporal causal discovery and structure learning with attention-based convolutional neural networks.
\newblock Master's thesis, University of Twente, 2018.

\bibitem[Zhu et~al.(2017)Zhu, Zhang, Zhang, and Liu]{zhu2017diverse}
Guibo Zhu, Zhaoxiang Zhang, Xu-Yao Zhang, and Cheng-Lin Liu.
\newblock Diverse neuron type selection for convolutional neural networks.
\newblock In \emph{IJCAI}, pages 3560--3566, 2017.

\bibitem[Kayhan and Gemert(2020)]{kayhan2020translation}
Osman~Semih Kayhan and Jan C~van Gemert.
\newblock On translation invariance in cnns: Convolutional layers can exploit absolute spatial location.
\newblock In \emph{Proceedings of the IEEE/CVF Conference on Computer Vision and Pattern Recognition}, pages 14274--14285, 2020.

\bibitem[Singh et~al.(2023)Singh, Singh, and Rana]{singh2023orthogonal}
Jaspreet Singh, Chandan Singh, and Ankur Rana.
\newblock Orthogonal transforms for learning invariant representations in equivariant neural networks.
\newblock In \emph{Proceedings of the IEEE/CVF Winter Conference on Applications of Computer Vision}, pages 1523--1530, 2023.

\bibitem[Colombo et~al.(2014)Colombo, Maathuis, et~al.]{colombo2014order}
Diego Colombo, Marloes~H Maathuis, et~al.
\newblock Order-independent constraint-based causal structure learning.
\newblock \emph{J. Mach. Learn. Res.}, 15\penalty0 (1):\penalty0 3741--3782, 2014.

\bibitem[Smith et~al.(2011)Smith, Miller, Salimi-Khorshidi, Webster, Beckmann, Nichols, Ramsey, and Woolrich]{smith2011network}
Stephen~M Smith, Karla~L Miller, Gholamreza Salimi-Khorshidi, Matthew Webster, Christian~F Beckmann, Thomas~E Nichols, Joseph~D Ramsey, and Mark~W Woolrich.
\newblock Network modelling methods for fmri.
\newblock \emph{Neuroimage}, 54\penalty0 (2):\penalty0 875--891, 2011.

\bibitem[Tank et~al.(2021)Tank, Covert, Foti, Shojaie, and Fox]{tank2021neural}
Alex Tank, Ian Covert, Nicholas Foti, Ali Shojaie, and Emily~B Fox.
\newblock Neural granger causality.
\newblock \emph{IEEE Transactions on Pattern Analysis and Machine Intelligence}, 44\penalty0 (8):\penalty0 4267--4279, 2021.

\bibitem[Wu et~al.(2021)Wu, Singh, and Berger]{wu2021granger}
Alexander~P Wu, Rohit Singh, and Bonnie Berger.
\newblock Granger causal inference on dags identifies genomic loci regulating transcription.
\newblock In \emph{International Conference on Learning Representations}, 2021.

\bibitem[Khanna and Tan(2019)]{khanna2019economy}
Saurabh Khanna and Vincent~YF Tan.
\newblock Economy statistical recurrent units for inferring nonlinear granger causality.
\newblock In \emph{International Conference on Learning Representations}, 2019.

\bibitem[Magliacane et~al.(2018)Magliacane, Van~Ommen, Claassen, Bongers, Versteeg, and Mooij]{magliacane2018domain}
Sara Magliacane, Thijs Van~Ommen, Tom Claassen, Stephan Bongers, Philip Versteeg, and Joris~M Mooij.
\newblock Domain adaptation by using causal inference to predict invariant conditional distributions.
\newblock \emph{Advances in neural information processing systems}, 31, 2018.

\bibitem[Rojas-Carulla et~al.(2018)Rojas-Carulla, Sch{\"o}lkopf, Turner, and Peters]{rojas2018invariant}
Mateo Rojas-Carulla, Bernhard Sch{\"o}lkopf, Richard Turner, and Jonas Peters.
\newblock Invariant models for causal transfer learning.
\newblock \emph{The Journal of Machine Learning Research}, 19\penalty0 (1):\penalty0 1309--1342, 2018.

\bibitem[Santos(2021)]{santos2021domain}
Luis Gustavo Moneda~dos Santos.
\newblock \emph{Domain generalization, invariance and the Time Robust Forest}.
\newblock PhD thesis, Universidade de S{\~a}o Paulo, 2021.

\bibitem[Li et~al.(2021)Li, Ao, and Mo]{li2021revisiting}
Zhenghui Li, Zhiming Ao, and Bin Mo.
\newblock Revisiting the valuable roles of global financial assets for international stock markets: Quantile coherence and causality-in-quantiles approaches.
\newblock \emph{Mathematics}, 9\penalty0 (15):\penalty0 1750, 2021.

\bibitem[Liu et~al.(2022)Liu, Cadei, Schweizer, Bahmani, and Alahi]{liu2022towards}
Yuejiang Liu, Riccardo Cadei, Jonas Schweizer, Sherwin Bahmani, and Alexandre Alahi.
\newblock Towards robust and adaptive motion forecasting: A causal representation perspective.
\newblock In \emph{Proceedings of the IEEE/CVF Conference on Computer Vision and Pattern Recognition}, pages 17081--17092, 2022.

\bibitem[Ahmad et~al.(2005)Ahmad, Ashraf, and Ahmed]{ahmad2005indian}
Khan~Masood Ahmad, Shahid Ashraf, and Shahid Ahmed.
\newblock Is the indian stock market integrated with the us and japanese markets? an empirical analysis.
\newblock \emph{South Asia Economic Journal}, 6\penalty0 (2):\penalty0 193--206, 2005.

\bibitem[Liang et~al.(2006)Liang, Huang, Saratchandran, and Sundararajan]{liang2006fast}
Nan-Ying Liang, Guang-Bin Huang, Paramasivan Saratchandran, and Narasimhan Sundararajan.
\newblock A fast and accurate online sequential learning algorithm for feedforward networks.
\newblock \emph{IEEE Transactions on neural networks}, 17\penalty0 (6):\penalty0 1411--1423, 2006.

\bibitem[Li et~al.(2014)Li, Zhang, Chen, and Smola]{li2014efficient}
Mu~Li, Tong Zhang, Yuqiang Chen, and Alexander~J Smola.
\newblock Efficient mini-batch training for stochastic optimization.
\newblock In \emph{Proceedings of the 20th ACM SIGKDD international conference on Knowledge discovery and data mining}, pages 661--670, 2014.

\bibitem[Hong et~al.(2020)Hong, Gao, Yao, Zhang, Plaza, and Chanussot]{hong2020graph}
Danfeng Hong, Lianru Gao, Jing Yao, Bing Zhang, Antonio Plaza, and Jocelyn Chanussot.
\newblock Graph convolutional networks for hyperspectral image classification.
\newblock \emph{IEEE Transactions on Geoscience and Remote Sensing}, 59\penalty0 (7):\penalty0 5966--5978, 2020.

\end{thebibliography}

\end{document}